\documentclass[manuscript,screen]{acmart}
\usepackage{utils}

\AtBeginDocument{%
  \providecommand\BibTeX{{%
    \normalfont B\kern-0.5em{\scshape i\kern-0.25em b}\kern-0.8em\TeX}}}

\setcopyright{acmcopyright}
\copyrightyear{2018}
\acmYear{2018}
\acmDOI{XXXXXXX.XXXXXXX}

\begin{document}

%%
%% The "title" command has an optional parameter,
%% allowing the author to define a "short title" to be used in page headers.
% \title{CMA-ES with Margin: Lower-Bounding Marginal Probability for Mixed-Integer Black-Box Optimization}
% \title{CMA-ES with Margin for Single-and Multi-Objective Mixed-Integer Black-Box Optimization}
\title{\revise{Marginal Probability-Based Integer Handling for CMA-ES Tackling Single-and Multi-Objective Mixed-Integer Black-Box Optimization}}

%%
%% The "author" command and its associated commands are used to define
%% the authors and their affiliations.
%% Of note is the shared affiliation of the first two authors, and the
%% "authornote" and "authornotemark" commands
%% used to denote shared contribution to the research.
\author{Ryoki Hamano}
% \authornote{Both authors contributed equally to this research.}
\affiliation{%
  \institution{Yokohama National University}
%   \streetaddress{}
  \city{Yokohama}
  \state{Kanagawa}
  \country{Japan}
%   \postcode{43017-6221}
}
\email{hamano-ryoki-pd@ynu.jp}

\author{Shota Saito}
\affiliation{%
  \institution{Yokohama National University \& SKILLUP NeXt, Ltd.}
%   \streetaddress{}
  \city{Yokohama}
  \state{Kanagawa}
  \country{Japan}
}
\email{saito-shota-bt@ynu.jp}

\author{Masahiro Nomura}
\affiliation{%
  \institution{CyberAgent}
  \city{Shibuya}
  \state{Tokyo}
  \country{Japan}
 }
\email{nomura\_masahiro@cyberagent.co.jp}

\author{Shinichi Shirakawa}
\affiliation{%
  \institution{Yokohama National University}
  \city{Yokohama}
  \state{Kanagawa}
  \country{Japan}
}
\email{shirakawa-shinichi-bg@ynu.ac.jp}

%%
%% By default, the full list of authors will be used in the page
%% headers. Often, this list is too long, and will overlap
%% other information printed in the page headers. This command allows
%% the author to define a more concise list
%% of authors' names for this purpose.
\renewcommand{\shortauthors}{Hamano and Saito et al.}

%%
%% The abstract is a short summary of the work to be presented in the
%% article.
\begin{abstract}
This study targets the mixed-integer black-box optimization (MI-BBO) problem where continuous and integer variables should be optimized simultaneously. The CMA-ES, our focus in this study, is a population-based stochastic search method that samples solution candidates from a multivariate Gaussian distribution (MGD), which shows excellent performance in continuous BBO. The parameters of MGD, mean and (co)variance, are updated based on the evaluation value of candidate solutions in the CMA-ES. If the CMA-ES is applied to the MI-BBO with straightforward discretization, however, the variance corresponding to the integer variables becomes much smaller than the granularity of the discretization before reaching the optimal solution, which leads to the stagnation of the optimization.
In particular, when binary variables are included in the problem, this stagnation more likely occurs because the granularity of the discretization becomes wider, and the existing integer handling for the CMA-ES does not address this stagnation.
To overcome these limitations, we propose a simple integer handling for the CMA-ES based on lower-bounding the marginal probabilities associated with the generation of integer variables in the MGD. The numerical experiments on the MI-BBO benchmark problems demonstrate the efficiency and robustness of the proposed method. Furthermore, in order to demonstrate the generality of the idea of the proposed method, in addition to the single-objective optimization case, we incorporate it into multi-objective CMA-ES and verify its performance on bi-objective mixed-integer benchmark problems.
\end{abstract}

%%
%% The code below is generated by the tool at http://dl.acm.org/ccs.cfm.
%% Please copy and paste the code instead of the example below.
%%

\begin{CCSXML}
<ccs2012>
   <concept>
       <concept_id>10002950.10003714.10003716.10011141</concept_id>
       <concept_desc>Mathematics of computing~Mixed discrete-continuous optimization</concept_desc>
       <concept_significance>500</concept_significance>
       </concept>
 </ccs2012>
\end{CCSXML}

\ccsdesc[500]{Mathematics of computing~Mixed discrete-continuous optimization}

%%
%% Keywords. The author(s) should pick words that accurately describe
%% the work being presented. Separate the keywords with commas.
\keywords{covariance matrix adaptation evolution strategy, mixed-integer black-box optimization}

%%
%% This command processes the author and affiliation and title
%% information and builds the first part of the formatted document.
\maketitle

\section{Introduction}
The mixed-integer black-box optimization (MI-BBO) problem is the problem of simultaneously optimizing continuous and integer variables under the condition that the objective function is not explicitly given, which does not allow full access to the gradient.
The MI-BBO problems often appear in real-world applications\footnote{In these applications, the decision variables may include categorical variables, but note that in MI-BBO, the decision variables consist only of continuous and integer variables. In particular, this paper does not deal with categorical variables, except for binary variables, as they cannot necessarily be ordered meaningfully, unlike continuous and integer variables.} such as, material design~\cite{zhang_bayesian_2020, iyer_data_2020}, topology optimization~\cite{yang_evolutionary_1998, fujii_cma-es-based_2018}, placement optimization for $\mathrm{CO_{2}}$ capture and storage~\cite{Miyagi_GECCO2018}, and hyper-parameter optimization of machine learning~\cite{hutter_automated_2019,hazan2018hyperparameter}.
Several algorithms have been designed for MI-BBO so far, e.g., the extended evolution strategies~\cite{li_mixed_2013, MIES1995, MIES2010} and surrogate model-based method~\cite{Bliek_GECCO21}.
Although various BBO methods have been proposed for continuous domain~\cite{Rechenberg:1973, hansen_adapting_1996}, integer domain~\cite{IntEA1994}, and binary domain~\cite{Holland1975AdaptationIN, baluja_population-based_1994}, the BBO methods for mixed-integer problems have not been actively developed.
One common way is applying a continuous BBO method to an MI-BBO problem by discretizing the continuous variables when evaluating a candidate solution rather than using a specialized method for MI-BBO.

The \textit{covariance matrix adaptation evolution strategy} (CMA-ES) \cite{hansen_adapting_1996, hansen_reducing_2003} is a powerful method in continuous black-box optimization, aiming to minimize the objective function through population-based stochastic search. The CMA-ES samples several continuous vectors from a multivariate Gaussian distribution (MGD) and then evaluates the objective function values of the vectors. Subsequently, the CMA-ES facilitates optimization by updating the mean vector, covariance matrix, and step-size (overall standard deviation) based on evaluation values. The CMA-ES exhibits two attractive properties for users. First, it has several invariance properties, such as the invariance of a strictly monotonic transformation of the objective function and an affine transformation (rotation and translation) of the search space. These invariances make it possible to generalize the powerful empirical performance of the CMA-ES of one particular problem to another. Second, the CMA-ES is a quasi-parameter-free algorithm, which allows use without tuning the hyperparameters, such as the learning rate for the mean vector or covariance matrix. In the CMA-ES, all hyperparameters are given default values based on theoretical work and careful experiments.

The most straightforward way to apply the CMA-ES to the MI-BBO is to discretize some elements of the sampled continuous vector, e.g., \cite{TAMILSELVI2014208}.
This approach transforms the original mixed-integer function into a continuous function with a plateau by relaxing the integer variables into the continuous variables.
However, because of the plateau, this simple method may not change the evaluation value by small variations in elements corresponding to the integer variables. Specifically, as pointed out in \cite{hansen_cma-es_2011}, this stagnation occurs when the sample standard deviation of the dimension corresponding to an integer variable becomes much smaller than the granularity of the discretization. In particular, the step-size tends to decrease with each iteration, which promotes trapping on the plateau of the integer variables.
To address this plateau problem in the integer variable treatment, \citet{hansen_cma-es_2011} proposed the injection of mutations into a sample of elements corresponding to integer variables in the CMA-ES, and \citet{Miyagi_GECCO2018} used this modification for the real-world MI-BBO problem.
Although this mutation injection is effective on certain problem classes, \citet{hansen_cma-es_2011} mentioned that it is not suitable for binary variables or $k$-ary integers in $k < 10$.

This study aims to improve the integer variable treatment of the CMA-ES in the MI-BBO problem. First, we investigate why the CMA-ES search fails in MI-BBO problems involving binary variables. Following the result, we propose the adaptation of the sample discretization process according to the current MGD parameters. The proposed adaptive discretization process can be represented as an affine transformation for a sample. Therefore, owing to the affine invariance of CMA-ES, it is expected to maintain the good behavior of the original CMA-ES. Additionally, we extend the proposed method from binary variables to integer variables.

The proposed method, termed CMA-ES with Margin, has the advantage of high generality because it is based on a simple and principled idea, lower-bounding the marginal probabilities associated with the generation of integer variables in the MGD.
To demonstrate the generality of the idea of the proposed method, in addition to the single-objective optimization case, we also develop an extension of the CMA-ES with Margin for multi-objective optimization.
Specifically, we introduce the idea of CMA-ES with Margin into the MO-CMA-ES (multi-objective covariance matrix adaptation)~\cite{igel2007covariance,voss2010improved}.
This part is the extension of the previous work~\cite{hamano2022cma}.

The rest of this paper is organized as follows.
In Section~\ref{sec:cma-es}, we describe the CMA-ES and the existing method of mixed-integer handling~\cite{hansen_cma-es_2011}.
In Section~\ref{sec:preliminary}, we conduct numerical experiments to clarify why the existing mixed-integer handling is difficult to optimize binary variables.
In Section~\ref{sec:proposed}, we propose a simple extension of the CMA-ES for the MI-BBO.
In Section~\ref{sec:exp_and_res}, we apply the proposed method to the MI-BBO problems to validate its robustness and efficiency.
Furthermore, to demonstrate the generality of the idea of the proposed method, in Section~\ref{sec:multiobjective}, we develop an extension of the CMA-ES with Margin for multi-objective optimization, and we empirically analyze its behavior on bi-objective benchmark problems in Section~\ref{sec:exp_multiobjective}.
Section~\ref{sec:conclusion} concludes with summary and future direction of this work.

\section{CMA-ES and Mixed Integer Handling}
\label{sec:cma-es}
\subsection{CMA-ES}
Let us consider the black-box minimization problem in the continuous search space for an objective function $f: \R^{N} \rightarrow \R$. The CMA-ES samples an $N$-dimensional candidate solution $\x \in \R^{N}$ from an MGD $\mathcal{N}(\boldsymbol{m}, \sigma^2 \boldsymbol{C})$ parameterized by the mean vector $\boldsymbol{m} \in \R^{N}$, covariance matrix $\boldsymbol{C} \in \R^{N \times N}$, and step-size $\sigma \in \R_{> 0}$. The CMA-ES updates the distribution parameters based on the objective function value $f(\x)$. There are several variations in the update methods of distribution parameters, although we consider the de facto standard CMA-ES \cite{hansen2016cma}, which combines the weighted recombination, cumulative step-size adaptation (CSA), rank-one covariance matrix update, and rank-$\mu$ update. We use the default parameters proposed in \cite{hansen2016cma} and list them in Table \ref{tb:cma_params}. The CMA-ES repeats the following steps until a termination criterion is satisfied.

\begin{table}[t]
    \centering
    \caption{Default hyperparameters and initial values of the CMA-ES.} \vspace{-2.5mm}
    \begin{tabular}{c|c}
        \hline
        $\lambda$ & 4 + $\lfloor 3 \ln(N) \rfloor$ \\
        $\mu$ & $\lfloor \frac{\lambda }{ 2 } \rfloor$ \\
        \hline
        $w'_{i}$ & $ \ln\left( \frac{\lambda + 1}{2} \right) - \ln i$ \\
        $w_{i} (i \leq \mu)$ & $\frac{w'_{i}}{ \sum_{j=1}^\mu w'_{i} }$ \\
        $w_{i} (i > \mu)$ & $\frac{ w'_{i} }{\sum_{j=\mu + 1}^\lambda | w'_{i} | } \min \left( 1 + \frac{c_1}{c_\mu}, 1 + \frac{2 \muw^{-}}{\muw + 2}, \frac{1 - c_1 - c_\mu}{N c_\mu} \right)$  \\
        \hline
        $\muw$ & $\frac{1}{ \sum_{j=1}^\mu (w'_{i})^2 }$ \\
        $\muw^{-}$ & $\frac{\left(\sum_{j=\mu + 1}^\lambda w'_{i} \right)^2 }{ \sum_{j=\mu + 1}^\lambda (w'_{i})^2 }$ \\
        \hline
        $c_m$ & 1 \\
        $c_{\sigma}$ & $\frac{\muw + 2}{N + \muw + 5}$ \\
        $c_{c}$ & $\frac{4 + \muw / N}{N + 4 + 2\muw / N}$ \\
        $c_{1}$ & $\frac{2}{(N + 1.3)^2 + \muw }$ \\
        $c_{\mu}$ & $\min \left( 1 - c_1 , \frac{2(\muw - 2 + 1 / \muw)}{(N + 2)^2 + \muw} \right)$ \\
        $d_{\sigma}$ & $1 + c_{\sigma}+ 2 \max\left(0, \sqrt{\frac{\muw - 1}{N + 1}} - 1 \right)$ \\
        \hline
        $\ps[0]$, $\pc[0]$ & $\boldsymbol{0}$ \\
        $\m[0], \C[0], \sig[0]$ & Depending on the problem \\
        \hline
    \end{tabular}
    \label{tb:cma_params}
\end{table}

\paragraph{Sample and Evaluate Candidate Solutions}
In the $t$-th iteration,  the $\lambda$ candidate solutions $\boldsymbol{x}_{i}$ ($i = 1, 2, \dots, \lambda$) are sampled independently from the MGD $\mathcal{N}(\m[t], (\sig[t])^{2} \C[t])$ as follows:
\begin{align}
    \y_i = (\C[t])^{\frac{1}{2}} \boldsymbol{\xi}_{i} \enspace, \\
    \x_i = \m[t] + \sig[t] \y_i \enspace,
\end{align}
where $\boldsymbol{\xi}_i \sim \mathcal{N}(\boldsymbol{0}, \boldsymbol{I})$ represents a random vector with zero mean and a covariance matrix of the identity matrix $\boldsymbol{I}$, and $(\C[t])^{\frac{1}{2}}$ is the square root of the covariance matrix $\C[t]$ that is the symmetric and positive definite matrix satisfying $\C[t] = (\C[t])^{\frac{1}{2}} (\C[t])^{\frac{1}{2}}$.
The candidate solutions $\{\x_1, \x_2 \dots, x_{\lambda}\}$ are evaluated by $f$ and sorted by ranking. Let $x_{i:\lambda}$ be the $i$-th best candidate solution; then, $f(\x_{1:\lambda}) \leq f(\x_{2:\lambda}) \leq \dots \leq f(\x_{\lambda:\lambda})$ and let $\y_{i:\lambda}$ be the random vector corresponding to $\x_{i:\lambda}$. 

\paragraph{Update Mean Vector} The mean vector update uses the weighted sum of the best $\mu < \lambda$ candidate solutions and updates $\m[t]$ as follows:
\begin{align}
    \m[t+1] = \m[t] + c_{m} \sum_{i=1}^{\mu} w_{i} ( \x_{i:\lambda} - \m[t] ) \enspace, \label{eq:cma-es-m}
\end{align}
where $c_{m}$ is the learning rate for the mean vector, and the weight $w_{i}$ satisfies $w_{1} \geq w_{2} \geq \dots \geq w_{\mu} > 0$ and $\sum_{i=1}^{\mu} w_{i} = 1$.

\paragraph{Compute Evolution Paths}
For the step-size adaptation and the rank-one update of the covariance matrix, we use evolution paths that accumulate an exponentially fading pathway of the mean vector in the generation sequence. Let  $\boldsymbol{p}_{\sigma}$ and $\boldsymbol{p}_{c}$ describe the evolution paths for the step-size adaptation and rank-one update, respectively; then, $\ps[t]$ and $\pc[t]$ are updated as follows:
\begin{align}
    \ps[t+1] &= (1-c_\sigma)\ps[t] + \sqrt{c_\sigma(2-c_\sigma)\muw} {\C[t]}^{-\frac12} \sum_{i=1}^\mu w_i \y_{i:\lambda}  \enspace, \\
    \pc[t+1] &= (1-c_c) \pc[t] + h_\sigma \sqrt{c_c(2-c_c)\muw} \sum_{i=1}^\mu w_i \y_{i:\lambda} \enspace,
\end{align}
where $c_\sigma$ and $c_c$ are cumulative rates, $\muw$ is effective sample-size, and
{\small
\begin{align*}
    h_\sigma = \mathds{1} \left\{ \|\ps[t+1]\| < \sqrt{1-(1-c_\sigma)^{2(t+1)}}\left(1.4+\frac{2}{N+1}\right)\E [\|\mathcal{N}(\boldsymbol{0}, \boldsymbol{I}) \|] \right\}
\end{align*}
}is an indicator function used to suppress a rapid increase in $\boldsymbol{p}_{c}$, where $\E [\|\mathcal{N}(\boldsymbol{0}, \boldsymbol{I}) \|] \approx \sqrt{N} \left(1 - \frac{1}{4N} + \frac{1}{21N^2}\right)$ is the expected Euclidean norm of the sample from a standard Gaussian distribution.

\paragraph{Update Step-size and Covariance Matrix}
Using the evolution paths computed in the previous step, we update $\C[t]$ and $\sig[t]$ as follows:
\begin{align}
        \C[t+1] &= \left(1 - c_1 - c_\mu \sum^{\lambda}_{i=1} w_{i} + (1-h_\sigma)c_1 c_c(2-c_c) \right) \C[t] + \underbrace{c_1 \pc[t+1]{\pc[t+1]}^\top}_{\text{rank-one update}} + \underbrace{c_\mu \sum_{i=1}^\lambda w_i^{\circ} \y_{i:\lambda}\y_{i:\lambda}^\top}_{\text{rank-}\mu\text{ update}}  \enspace, \\
        \sig[t+1] &= \sig[t] \exp \left( \frac{c_\sigma}{d_\sigma} \left( \frac{\|\ps[t+1] \|}{\E [\|\mathcal{N}(\boldsymbol{0}, \boldsymbol{I}) \|]} - 1 \right) \right)  \enspace,
        \label{eq:cma-es-sig}
\end{align}
where
$w_i^{\circ} := w_i \cdot \left(1 \text{ if } w_i\ge0 \text{ else } N / \left\| (\C[t])^{-\frac{1}{2}} \boldsymbol{y}_{i:\lambda} \right\|^{2} \right)$, 
$c_1$ and $c_{\mu}$ are the learning rates for the rank-one and rank-$\mu$ updates, respectively. Additionally, $d_{\sigma}$ is a damping parameter for the step-size adaptation. Note that the covariance matrix $\C[t+1]$ is kept symmetric since $\C[t]$, $\pc[t+1]{\pc[t+1]}^\top$, and $\y_{i:\lambda}\y_{i:\lambda}^\top$ are symmetric matrices.

\subsection{CMA-ES with Mixed-Integer Handling}
\label{sec::mi-cma-es}
In \cite{hansen_cma-es_2011}, several steps of the CMA-ES are modified to handle the integer variables. To explain this modification, we apply notations $\lbrack \cdot \rbrack_{j}$ and $\langle \cdot \rangle_{j}$, where the former denotes the $j$-th element of an argument vector and the latter denotes the $j$-th diagonal element of an argument matrix. We denote the number of dimensions as $N = N_{\mathrm{co}} + N_{\mathrm{in}}$, where $N_{\mathrm{co}}$ and $N_{\mathrm{in}}$ are the numbers of the continuous and integer variables, respectively. More specifically, the 1st to $N_{\mathrm{co}}$-th elements and $(N_{\mathrm{co}} + 1)$-th to $N$-th elements of the candidate solution are the elements corresponding to the continuous and integer variables, respectively.

\paragraph{Inject Integer Mutation}
For the element corresponding to the integer variable, stagnation occurs when the sample standard deviation becomes much smaller than the granularity of the discretization. The main idea to solve this stagnation in \cite{hansen_cma-es_2011} is to inject the \emph{integer mutation vector} $\boldsymbol{r}^{\mathrm{int}}_{i} \in \mathbb{N}^{N}$ into the candidate solution, which is given by
\begin{align}
    \x_i = \m[t] + \sig[t] \y_i + \boldsymbol{S}^{\mathrm{int}} \boldsymbol{r}^{\mathrm{int}}_{i}  \enspace, 
    \label{eq::integer_mutation}
\end{align}
where $\boldsymbol{S}^{\mathrm{int}}$ is the diagonal matrix whose diagonal elements indicate the variable granularities, which is $\langle \boldsymbol{S}^{\mathrm{int}} \rangle_{j} = 1$ if $N_{\mathrm{co}} + 1 \leq j \leq N$; otherwise $\langle \boldsymbol{S}^{\mathrm{int}} \rangle_{j} = 0$ in usual case. The integer mutation vector $\boldsymbol{r}^{\mathrm{int}}_{i}$ is sampled as follows:
\begin{enumerate}[\textrm{Step }1\textrm{. }]
    \item Set up a randomly ordered set of elements indices $J^{(t)}$ satisfying $2 \sig[t] \langle \C[t] \rangle_{j}^{\frac{1}{2}} < \langle \boldsymbol{S}^{\mathrm{int}} \rangle_{j}$.
    \item Determine the number of candidate solutions into which the integer mutation is injected as follows:
    \begin{align*}
    \lambda_{\mathrm{int}}^{(t)} = \left\{
        \begin{array}{lll}
            0 & (|J^{(t)}| = 0) \\
            \min(\lambda / 10 + |J^{(t)}| + 1, \lfloor \lambda / 2 \rfloor - 1) & (0 < |J^{(t)}| < N) \\
            \lfloor \lambda / 2 \rfloor & (|J^{(t)}| = N)
        \end{array}
        \right.  \enspace .
    \end{align*}
    \item $\lbrack \boldsymbol{R}'_{i} \rbrack_{j} = 1$ if the element indicate $j$ is equal to mod$(i - 1,|J^{(t)}|)$-th element of $J^{(t)}$, otherwise $\lbrack \boldsymbol{R}'_{i} \rbrack_{j} = 0$.
    \item $\lbrack \boldsymbol{R}''_{i} \rbrack_{j}$ is sampled from a geometric distribution with the probability parameter $p = 0.7^{\frac{1}{|J^{(t)}|}}$ if $j \in J^{(t)}$, otherwise $\lbrack \boldsymbol{R}''_{i} \rbrack_{j} = 0$.
    \item $\boldsymbol{r}^{\mathrm{int}}_{i} = \pm ( \boldsymbol{R}'_{i} + \boldsymbol{R}''_{i})$ with the sign-switching probability 1/2 if $i \leq \lambda_{\mathrm{int}}^{(t)}$, otherwise $\boldsymbol{r}^{\mathrm{int}}_{i} = \boldsymbol{0}$.
    \item If $\lambda_{\mathrm{int}}^{(t)} > 0$, $\lbrack  \boldsymbol{r}^{\mathrm{int}}_{\lambda} \rbrack_{j} = \pm \left( \left\lfloor \frac{\lbrack \boldsymbol{x}_{1:\lambda}^{(t-1)} \rbrack_{j}}{\langle \boldsymbol{S}^{\mathrm{int}} \rangle_{j}} \right\rfloor - \left\lfloor \frac{\lbrack \m[t] \rbrack_{j}}{\langle \boldsymbol{S}^{\mathrm{int}} \rangle_{j}} \right\rfloor \right)$ with the sign-switching probability 1/2 if $\langle \boldsymbol{S}^{\mathrm{int}} \rangle_{j} > 0$, otherwise $\lbrack \boldsymbol{r}^{\mathrm{int}}_{\lambda} \rbrack_{j} = 0$. This is a modified version of \cite{hansen_cma-es_2011} and introduced in~\cite{Miyagi_GECCO2018}.
\end{enumerate}

\paragraph{Modify Step-size Adaptation}
If the standard deviation of the elements corresponding to the integer variables is much smaller than the granularity of the discretization, then the step-size adaptation rapidly decreases the step-size. To address this problem, \cite{hansen_cma-es_2011} proposed a modification of the step-size adaptation to remove the elements corresponding to integer variables with considerably smaller standard deviations from the evolution path $\ps[t+1]$ when updating the step-size as follows: 
\begin{align}
    \sig[t+1] &= \sig[t] \exp \left( \frac{c_\sigma}{d_\sigma} \left( \frac{\| \boldsymbol{I}_{\sigma}^{(t+1)} \ps[t+1] \|}{\E [\|\mathcal{N}(\boldsymbol{0}, \boldsymbol{I}_{\sigma}^{(t+1)} ) \|]} - 1 \right) \right)  \enspace,
\end{align}
where $\boldsymbol{I}_{\sigma}^{(t+1)}$ is the diagonal masking matrix, and $\langle \boldsymbol{I}_{\sigma}^{(t+1)} \rangle_{j} = 0$ if $5 \sigma \langle \C[t] \rangle_{j}^{\frac{1}{2}} / \sqrt{c_{\sigma}} < \langle \boldsymbol{S}^{\mathrm{int}} \rangle_{j}$; otherwise, $\langle \boldsymbol{I}_{\sigma}^{(t+1)} \rangle_{j} = 1$. The expected value $\|\mathcal{N}(\boldsymbol{0}, \boldsymbol{I}_{\sigma}^{(t+1)} ) \|$ is approximated by $\sqrt{M} \left(1 - \frac{1}{4M} + \frac{1}{21M^2}\right)$, where $M$ is the number of non-zero diagonal elements for $\boldsymbol{I}_{\sigma}^{(t+1)}$.

\subsection{Preliminary Experiment: Why Is It Difficult to Optimize Binary Variables with CMA-ES Introducing Integer Mutations?}
\label{sec:preliminary}

It is known that the integer variable handling of CMA-ES \cite{hansen_cma-es_2011} does not work well for binary variables.
However, the reasons for this have not been well explored. We then empirically check why this integer handling fails to optimize binary variables.

We consider the function $\textsc{Encoding}_f(\x_i)$ to binarize the elements of the candidate solution corresponding to the binary variables, and the dimension $N = N_{\mathrm{co}} + N_{\mathrm{bi}}$, where $N_{\mathrm{bi}}$ is the number of binary variables. We define $\textsc{Encoding}_f(\x_i): \mathbb{R}^{N_{\mathrm{co}}} \times \mathbb{R}^{N_{\mathrm{bi}}} \mapsto \mathbb{R}^{N_{\mathrm{co}}} \times \{0, 1 \}^{N_{\mathrm{bi}}}$ as
\begin{align}
    \label{eq::binary_encoding}
    \textsc{Encoding}_f(\x_i) =
    \left\{
        \begin{array}{lll}
            \lbrack \x_i \rbrack_{j} &  (1 \leq j \leq N_{\mathrm{co}}) \\
            \mathds{1} \{ \lbrack \x_i \rbrack_{j} > 0 \} & (N_{\mathrm{co}} + 1 \leq j \leq N)
        \end{array}
    \right.  \enspace .
\end{align}
The partially discretized candidate solution obtained by \eqref{eq::binary_encoding} is denoted by $\bar{\x}_i = \textsc{Encoding}_f(\x_i)$.

Compared to the integer variables, binary variables have a much wider interval, where the same binary variables can be taken after binarization. Therefore, if the variance decreases while the mean vector is so far from the threshold zero at which the binary variable changes, the optimization of the binary variable fails.

\paragraph{Settings}
We use the \textsc{SphereOneMax} function as the objective function, which is a combination of the \textsc{Sphere} function and the \textsc{OneMax} function for the continuous and binary variables, respectively. The \textsc{SphereOneMax} function is defined as
\begin{align}
    \textsc{SphereOneMax}(\bar{\x}_i) = \sum_{j=1}^{N_{\mathrm{co}}} \lbrack \bar{\x}_i\rbrack_j^2 + N_{\mathrm{bi}} - \sum_{k=N_{\mathrm{co}}+1}^N \lbrack \bar{\x}_i \rbrack_k  \enspace, 
\end{align}
where the optimal solution 
is 0 for continuous variables and 1 for binary variables, respectively, and $\textsc{SphereOneMax}(\bar{\x}^{*}) = 0$. We check the behavior using the CMA-ES with integer variable handling introduced in Section \ref{sec::mi-cma-es} . Additionally, we use this CMA-ES variant with a box constraint $\lbrack \x_i \rbrack_{j} \in \lbrack -1 , 1 \rbrack$ corresponding to the binary variables. When the box constraint is used, the penalty $\| \x_{i}^{\mathrm{feas}} - \x_{i} \|^{2}_{2} / N$ is added to the evaluation value, where $\x_{i}^{\mathrm{feas}}$ is the nearest-neighbor feasible solution to $\x_{i}$. The number of dimensions $N$ is set to 40, and $N_{\mathrm{co}} = N_{\mathrm{bi}} = N/2 = 20$.

The initial mean vector $\m[0]$ is set to uniform random values in the range $\lbrack 1, 3 \rbrack$ for continuous variables and 0 for the binary variables, respectively. The covariance matrix and step-size are initialized with $\C[0] = \boldsymbol{I}$ and $\sig[0] = 1$, respectively. The optimization is successful when the best-evaluated value is less than $10^{-10}$, and the optimization is stopped when the minimum eigenvalue of $\sigma^{2} \boldsymbol{C}$ is less than $10^{-30}$. 

\begin{figure*}[t]
    \centering
    \includegraphics[width=\linewidth]{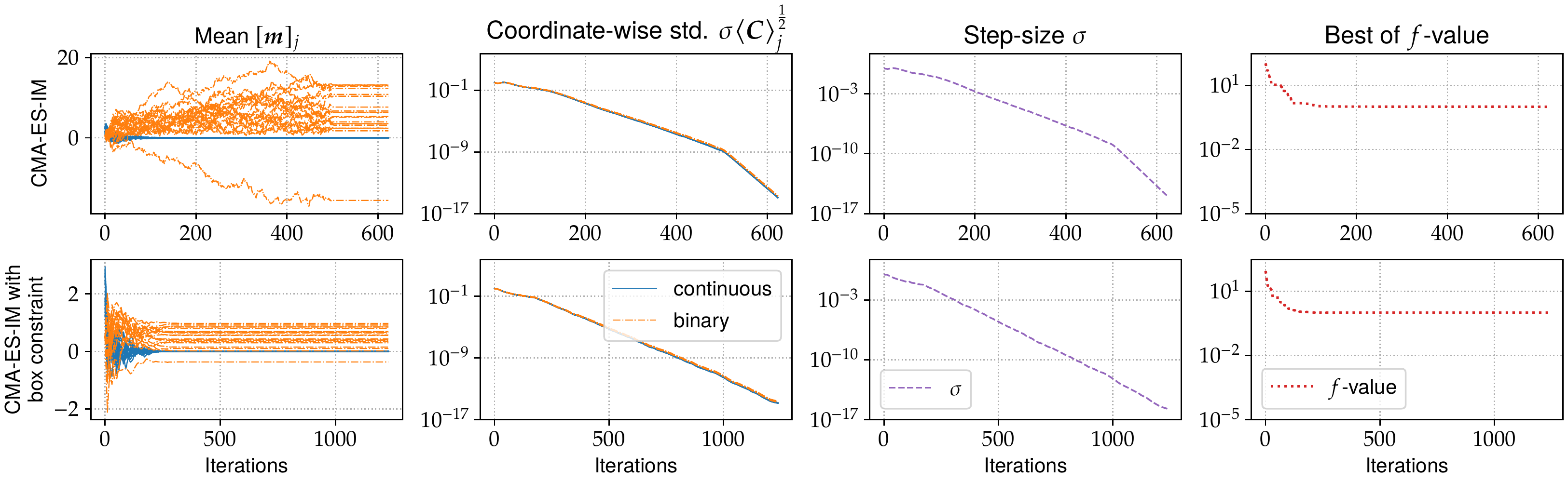}
    \caption{Transition of each element of the mean vector and the diagonal elements of the covariance matrix on CMA-ES-IM \cite{hansen_cma-es_2011} with and without the box constraint for a typical single failed trial on 40-dimensional \textsc{SphereOneMax}.}
    \label{fig::transition_m_and_C}
\end{figure*}

\paragraph{Result and Discussion}
For the coordinate-wise mean $\lbrack \boldsymbol{m} \rbrack_{j}$, the coordinate-wise standard deviation $\sigma \langle \boldsymbol{C} \rangle_{j}^{\frac{1}{2}}$, step-size $\sigma$, and best-evaluated value, the upper and lower sides of the Figure \ref{fig::transition_m_and_C} show the transitions of a single typical run of the optimization failure for the CMA-ES with the integer mutation and modification of the step-size adaptation (denoted by CMA-ES-IM) and the CMA-ES-IM with the box constraint. 
The CMA-ES-IM decreases the coordinate-wise standard deviations for binary variables with the step-size. In contrast, coordinate-wise mean is far from the threshold value of zero.
In this case, the integer mutation provided in Step~1 to Step~5 in Section~\ref{sec::mi-cma-es} is not effective to improve the evaluation value because in the dimension corresponding to the binary variable, $\boldsymbol{S}^\mathrm{int} \boldsymbol{r}^\mathrm{int}_i $~($i = 1, \ldots, \lambda_\mathrm{int}^{(t)}$) are smaller than the distances between $\m[t] + \sig[t]\boldsymbol{y}_i$ and the threshold value of zero. Moreover, when $\boldsymbol{S}^\mathrm{int} \boldsymbol{r}^\mathrm{int}_\lambda$ calculated in Step~6 also becomes small, the mutation no longer affects the candidate solutions at all~(after 500 iterations in Figure~\ref{fig::transition_m_and_C}).

On the other hand, CMA-ES-IM with box constraint can prevent the coordinate-wise mean from being far from zero and avoid fixation of candidate solutions by the mutation.
However, even if the mutation works, a high penalty value of box constraint results in a poor evaluation value. In this case, the mutated samples cannot be reflected in the mean update, which uses only the superior $\mu$ samples.
Then, the stagnation problem in the negative domain still remains.

These results suggest that we need a new way to handle integer variables that takes binary variables into account instead of the integer mutation. In Section \ref{sec:proposed}, we propose an integer handling method that preserves the generation probability of a different integer variable by introducing a correction that brings the coordinate-wise mean closer to a threshold value as the coordinate-wise standard deviation decreases.
\section{Proposed Method}
\label{sec:proposed}
In this section, we propose a simple modification of the CMA-ES in the MI-BBO. The basic idea is to introduce a lower bound on the marginal probability referred to as the \textit{margin}, so that the sample is not fixed to a single integer variable. The margin is a common technique in 
the estimation of distribution algorithms~(EDAs) for binary domains to address the problem of bits being fixed to 0 or 1. In fact, the population-based incremental learning~(PBIL)~\cite{PBIL:1994}, a binary variable optimization method based on Bernoulli distribution, restricts the updated marginals to the range $[1/N, 1-1/N]$. This prevents the optimization from stagnating with the distribution converging to an undesirable direction before finding the optimum.

To introduce this margin correction to the CMA-ES, we define a diagonal matrix $\boldsymbol{A}$ whose initial value is given by the identity matrix and redefine the MGD that generates the samples as $\mathcal{N}(\boldsymbol{m}, \sigma^2 \boldsymbol{A} \boldsymbol{C} \boldsymbol{A}^\top)$. The margin correction is achieved by correcting $\boldsymbol{A}$ and $\boldsymbol{m}$ so that the probability of the integer variables being generated outside the dominant values is maintained above a certain value $\alpha$. Because the sample generated from $\mathcal{N}(\boldsymbol{m}, \sigma^2 \boldsymbol{A} \boldsymbol{C} \boldsymbol{A}^\top)$ is equivalent to applying the affine transformation of $\boldsymbol{A}$ to the sample generated from $\mathcal{N}(\boldsymbol{m}, \sigma^2 \boldsymbol{C})$, we can separate the adaptation of the covariance and the update of $\boldsymbol{A}$. Consequently, the proposed modification can be represented as the affine transformation of the samples used to evaluate the objective function, without making any changes to the updates in CMA-ES.
It should be noted that although the mean vector can also be corrected by the affine transformation, we directly correct it to avoid the divergence of $\boldsymbol{m}$.

In this section, we first redefine $\textsc{Encoding}_f$ to facilitate the introduction of the margin. 
Next, we show the process of the CMA-ES with the proposed modification.
Finally, we explain the margin correction, namely, the updates of $\boldsymbol{A}$ and $\boldsymbol{m}$, separately for the cases of binary and integer variables. The code of the proposed CMA-ES with Margin is available at \url{https://github.com/EvoConJP/CMA-ES_with_Margin}.

\subsection{Definition of $\textsc{Encoding}_f$ and Threshold $\ell$}
\label{ssec:encoding_function}
Let $z_{j,k}$ be the $k$-th smallest value among the discrete values in the $j$-th dimension, where $N_{\mathrm{co}} + 1 \leq j \leq N$ and $1 \leq k \leq K_j$.
It should be noted that $K_j$ is the number of candidate integers for the $j$-th variable $\boldsymbol{z}_j$.
Under this definition, the binary variable can also be represented as, e.g. $z_{j,1}=0$, $z_{j,2}=1$. Moreover, we introduce a threshold $\thd$ for encoding continuous variables into discrete variables. Let $\ell_{j, k|k+1}$ be the threshold of two discrete variables $z_{j,k}$ and $z_{j,k+1}$; it is given by the midpoint of $z_{j,k}$ and $z_{j,k+1}$, namely, $\ell_{j, k|k+1} := (z_{j,k} + z_{j,k+1})/2$. We then redefine $\textsc{Encoding}_f$ when $N_{\mathrm{co}} + 1 \leq j \leq N$ as follows:
\begin{align*}
    \textsc{Encoding}_f([\x_i]_j) =
    \left\{
        \begin{array}{lll}
            z_{j,1} & \text{if} \enspace [\x_i]_j \leq \ell_{j, 1|2} \\
            z_{j,k} & \text{if} \enspace \ell_{j, k-1|k} < [\x_i]_j \leq \ell_{j, k|k+1} \\
            z_{j,K_j} & \text{if} \enspace \ell_{j, K_j-1|K_j} < [\x_i]_j
        \end{array}
    \right.
\end{align*}
Moreover, if $1 \leq j \leq N_{\mathrm{co}}$, $[\x_i]_j$ is isometrically mapped as $\textsc{Encoding}_f([\x_i]_j) = [\x_i]_j$.
Then, the discretized candidate solution is denoted by $\bar{\x}_i = \textsc{Encoding}_f(\x_i)$. The set of discrete variables $\boldsymbol{z}_j$ is not limited to consecutive integers such as $\{ 0,1,2 \}$, but can also handle general discrete variables such as $\{ 1,2,4 \}$ and $\{ 0.01,0.1,1 \}$.

\subsection{CMA-ES with the Proposed Modification}
Given $\A[0]$ as an identity matrix $\boldsymbol{I}$, the update of the proposed method, termed \textit{CMA-ES with margin}, at the iteration $t$ is given in the following steps.
\begin{enumerate}[\textrm{Step }1\textrm{. }]
    \item The $\lambda$ candidate solutions $\x_i$~($i = 1,2,\ldots, \lambda$) are sampled from $\mathcal{N} (\m[t], (\sig[t])^2 \C[t])$ as $\x_i = \m[t] + \sig[t] \y_i$, where $\y_i \sim \mathcal{N} (\boldsymbol{0}, \C[t])$ for $i = 1,2,\ldots, \lambda$.
    \item The affine transformed solutions $\boldsymbol{v}_i$~($i = 1,2,\ldots, \lambda$) are calculated as $\boldsymbol{v}_i = \m[t] + \sig[t] \A[t] \y_i$ for $i = 1,2,\ldots, \lambda$.
    \item The discretized $\boldsymbol{v}_i$, i.e., $\bar{\boldsymbol{v}}_i$~($i = 1,2,\ldots, \lambda$) are evaluated by $f$ and sort $\{\x_{1:\lambda}, \x_{2:\lambda}, \dots ,\x_{\lambda:\lambda}\}$ and $\{\y_{1:\lambda}, \y_{2:\lambda}, \dots ,\y_{\lambda:\lambda}\}$ so that the indices correspond to $f(\bar{\boldsymbol{v}}_{1:\lambda}) \leq f(\bar{\boldsymbol{v}}_{2:\lambda}) \leq \cdots \leq f(\bar{\boldsymbol{v}}_{\lambda:\lambda})$.
    \item Based on \eqref{eq:cma-es-m} to \eqref{eq:cma-es-sig}, update $\m[t]$, $\C[t]$, and $\sig[t]$ using $\x$ and $\y$.
    \item Modify $\m[t+1]$ and update $\A[t]$ based on Section~\ref{ssec:margin_bin} and Section~\ref{ssec:margin_int}. \label{step:modify}
\end{enumerate}
It should be noted that the algorithm based on the above is consistent with the original CMA-ES if no corrections are made in Step~\ref{step:modify}. In other words, the smaller the margin parameter $\alpha$, described in Section~\ref{ssec:margin_bin} and Section~\ref{ssec:margin_int}, and the more insignificant the modification, the closer the above algorithm is to the original CMA-ES. Moreover, the update of the variance-covariance has not been modified, which facilitates the smooth consideration of the introduction of the CMA-ES properties, e.g., step-size adaptation methods other than CSA.

\begin{figure}[t]
    \begin{minipage}[b]{0.45\linewidth}
    \begin{center}
      \includegraphics[width=\linewidth]{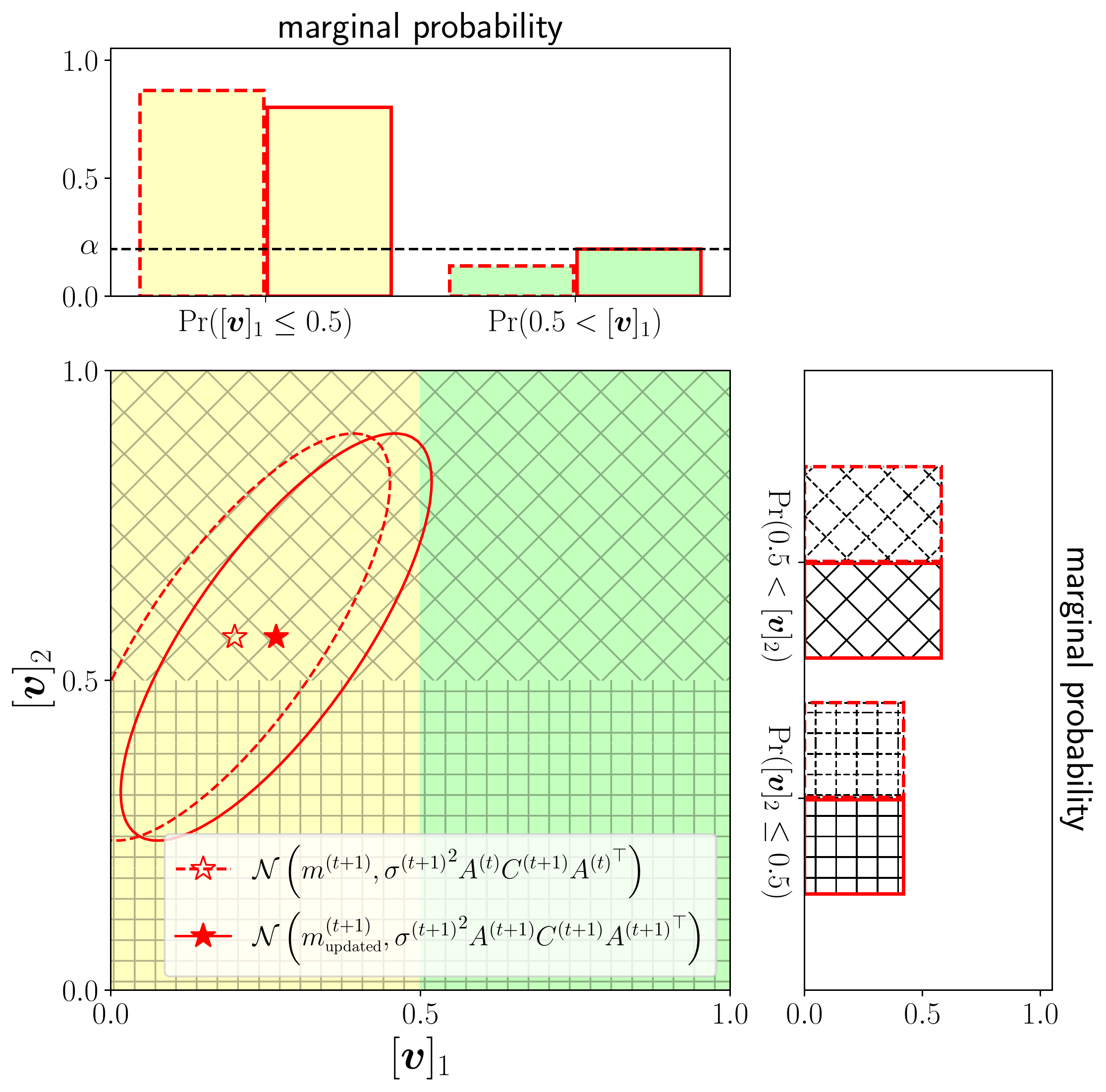}
    \end{center}
    \end{minipage}
    \hspace{3.0mm}
    \begin{minipage}[b]{0.45\linewidth}
    \begin{center}
      \includegraphics[width=\linewidth]{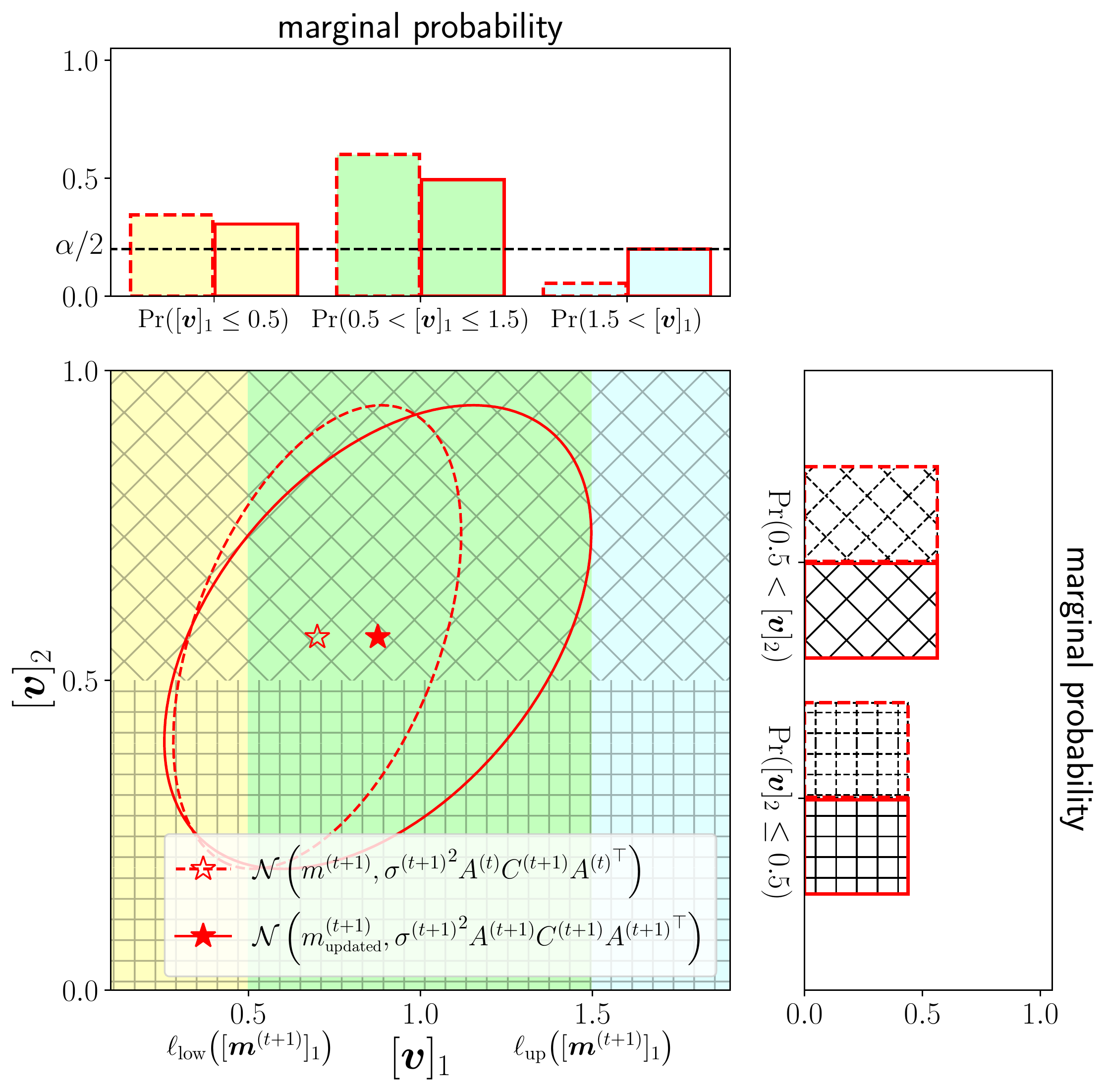}
    \end{center}
    \end{minipage}
    \vspace{-3.0mm}
    \caption{Example of MGD followed by $\boldsymbol{v}$ and its marginal probability. The dashed red ellipse corresponds to the MGD before the correction, whereas the solid one corresponds to the MGD after the correction for the binary variable (left) and for the integer variable (right).}
    \label{fig:correction}
\end{figure}
\subsection{Margin for Binary Variables} \label{ssec:margin_bin}
Considering the probability that a binarized variable $[\bar{\boldsymbol{v}}]_j$ is $0$ and the probability that it is $1$, the following conditions should be satisfied after the modification by the margin:
\begin{align}
    \min \left\{ \Pr([\bar{\boldsymbol{v}}]_j = 0), \Pr([\bar{\boldsymbol{v}}]_j = 1) \right\} \geq \alpha 
    \Leftrightarrow \min \left\{ \Pr([\boldsymbol{v}]_j < 0.5), \Pr([\boldsymbol{v}]_j \geq 0.5) \right\} \geq \alpha
\end{align}
It should be noted that the binarize threshold $\ell_{j,1|2}$ is equal to $0.5$. Here, the left of Figure~\ref{fig:correction} shows an example of the updated MGD followed by $\boldsymbol{v}$ and the marginal probabilities. The MGD before the margin correction~(dashed red ellipse) shows that its marginal probability $\Pr([\boldsymbol{v}]_1 \geq 0.5)$ is smaller than the margin parameter $\alpha$. In this case, by modifying the element of the mean vector $[\m[t+1]]_1$, we can correct the marginal probability $\Pr([\boldsymbol{v}]_1 \geq 0.5)$ to $\alpha$ without affecting the other dimensions. To calculate the amount of the correction for this mean vector, we consider the confidence interval of the probability $1-2\alpha$ in the marginal distribution of the $j$-th dimension. The confidence interval of the $j$-th dimension is represented by
\begin{align*}
    \left[ [\m[t+1]]_j - \textrm{CI}_j^{(t+1)} (1-2\alpha), [\m[t+1]]_j + \textrm{CI}_j^{(t+1)} (1-2\alpha) \right] \enspace.
\end{align*}
It should be noted that $\textrm{CI}_j^{(t+1)} (1-2\alpha)$ is defined as
\begin{align*}
    \textrm{CI}_j^{(t+1)} (1-2\alpha) := \sqrt{\chi^2_{\textrm{ppf}}(1-2\alpha) {\sig[t+1]}^2 \left\langle \A[t] \C[t+1] {\A[t]}^\top\right\rangle_j } \enspace,
\end{align*}
where $\chi^2_\textrm{ppf}(\cdot)$ is the quantile function that, given the lower cumulative probability, returns the percentage point in the chi-squared distribution with $1$ degree of freedom. If the threshold $\ell_{j,1|2} = 0.5$ is outside this confidence interval, the marginal probability to be corrected is less than $\alpha$. Given the encoding threshold closest to $[\m[t+1]]_j$ as $\ell\bigl([\m[t+1]]_j\bigr)$, which is equal to $0.5$ in the binary case, the modification for the $j$-th element of the mean vector can be denoted as
\begin{align}
    [\m[t+1]]_j \leftarrow \ell\left([\m[t+1]]_j\right) + &\sign\left( [\m[t+1]]_j - \ell\left([\m[t+1]]_j\right) \right) \notag \\
    &\quad \cdot \min \left\{ \left| [\m[t+1]]_j - \ell\left([\m[t+1]]_j\right) \right|, \textrm{CI}_j^{(t+1)} (1-2\alpha) \right\} \enspace. \label{eq:bin_correct_m}
\end{align}
Additionally, no changes are made to $\langle \A[t] \rangle_j$, namely,
\begin{align}
   \langle \A[t+1] \rangle_j \leftarrow \langle \A[t] \rangle_j \enspace. \label{eq:bin_correct_A}
\end{align}
As shown in the solid red line in the left of Figure~\ref{fig:correction}, the marginal probability after this modification is lower-bounded by $\alpha$.

\subsection{Margin for Integer Variables}
\label{ssec:margin_int}
First, we consider the cases where the $j$-th element of the mean vector satisfies $[\m[t+1]]_j \leq \ell_{j, 1|2}$ or $\ell_{j, K_j-1|K_j} < [\m[t+1]]_j$. In these cases, the integer variable $[\boldsymbol{v}]_j$ may be fixed to $z_{j,1}$ or $z_{j,K_j}$, respectively. Thus, we correct the marginal probability of generating one inner integer variable, i.e., $z_{j,2}$ for $z_{j,1}$ or $z_{j,K_j-1}$ for $z_{j,K_j}$, to maintain it above $\alpha$, respectively. This correction is achieved by updating $[\m[t+1]]_j$ and $\A[t]$ based on \eqref{eq:bin_correct_m} and \eqref{eq:bin_correct_A}.

Next, we consider the case of other integer variables. The right of Figure~\ref{fig:correction} shows an example of the updated MGD followed by $\boldsymbol{v}$ and the marginal probability. In this example, $[\boldsymbol{v}]_1$ is expected to be fixed in the interval $(0.5,1.5]$ when $\Pr([\boldsymbol{v}]_1 \leq 0.5)$ and $\Pr(1.5 < [\boldsymbol{v}]_1)$ become small. Thus, the correction strategy is to lower-bound the probability of $[\boldsymbol{v}]_j$ being generated outside the plateau where $[\boldsymbol{v}]_j$ is expected to be fixed, such as $\Pr([\boldsymbol{v}]_1 \leq 0.5)$ and $\Pr(1.5 < [\boldsymbol{v}]_1)$ in the right of Figure~\ref{fig:correction}. In this case, the value of the margin is set to $\alpha/2$.
For simplicity, we denote $\ell_\textrm{low}\bigl([\m[t+1]]_j\bigr)$ and $\ell_\textrm{up}\bigl([\m[t+1]]_j\bigr)$ as
\begin{align}
    \ell_\textrm{low}\left([\m[t+1]]_j\right) &:= \max \left\{ l \in \thd_j : l < [\m[t+1]]_j \right\} \enspace, \\
    \ell_\textrm{up}\left([\m[t+1]]_j\right) &:= \min \left\{ l \in \thd_j : [\m[t+1]]_j \leq l \right\} \enspace.
\end{align}
The first step of the modification is to calculate $p_\textrm{low}$, $p_\textrm{up}$, and $p_\textrm{mid}$ as follows.
\begin{align}
    p_\textrm{low} &\leftarrow \Pr\left([\boldsymbol{v}]_j \leq \ell_\textrm{low}\left([\m[t+1]]_j\right)\right) \\
    p_\textrm{up} &\leftarrow \Pr\left(\ell_\textrm{up}\left([\m[t+1]]_j \right) < [\boldsymbol{v}]_j \right) \\
    p_\textrm{mid} &\leftarrow 1 - p_\textrm{low} - p_\textrm{up}
\end{align}
Next, we restrict $p_\textrm{low}$, $p_\textrm{up}$, and $p_\textrm{mid}$ as follows.
\begin{align}
    p'_\textrm{low} &\leftarrow \max \{ \alpha/2 , p_\textrm{low} \} \label{eq:lower_low} \\
    p'_\textrm{up} &\leftarrow \max \{ \alpha/2 , p_\textrm{up} \} \label{eq:lower_up} \\
    p''_\textrm{low} &\leftarrow p'_\textrm{low} + \frac{1 - p'_\textrm{low} - p'_\textrm{up} - p_\textrm{mid}}{p'_\textrm{low} + p'_\textrm{up} + p_\textrm{mid} - 3 \cdot \alpha/2}(p'_\textrm{low} - \alpha/2) \label{eq:ensure_low} \\
    p''_\textrm{up} &\leftarrow p'_\textrm{up} + \frac{1 - p'_\textrm{low} - p'_\textrm{up} - p_\textrm{mid}}{p'_\textrm{low} + p'_\textrm{up} + p_\textrm{mid} - 3 \cdot \alpha/2}(p'_\textrm{up} - \alpha/2) \label{eq:ensure_up}
\end{align}
The equations \eqref{eq:ensure_low} and \eqref{eq:ensure_up} ensure $p''_\textrm{low} + p''_\textrm{up} + p'_\textrm{mid} = 1$, while keeping $p''_\textrm{low} \geq \alpha/2$ and $p''_\textrm{up} \geq \alpha/2$, where $p'_\textrm{mid}=1-p''_\textrm{low} - p''_\textrm{up}$. This handling method is also adopted in~\cite[Appendix~D]{ASNG:2019}.
We update $\m[t+1]$ and $\A[t]$ so that the corrected marginal probabilities $\Pr\bigl([\boldsymbol{v}]_j \leq \ell_\textrm{low}\bigl([\m[t+1]]_j\bigr)\bigr)$ and $\Pr\bigl(\ell_\textrm{up}\bigl([\m[t+1]]_j \bigr) < [\boldsymbol{v}]_j \bigr)$ are $p''_\textrm{low}$ and $p''_\textrm{up}$, respectively. The conditions to be satisfied are as follows.
\begin{align}
    \begin{cases}
        [\m[t+1]]_j - \ell_\textrm{low}\left( [\m[t+1]]_j \right) = \textrm{CI}_j^{(t+1)} (1 - 2p''_\textrm{low}) \\
        \ell_\textrm{up}\left( [\m[t+1]]_j \right) - [\m[t+1]]_j = \textrm{CI}_j^{(t+1)} (1 - 2p''_\textrm{up})
    \end{cases} \label{eq:simult}
\end{align}
\begin{algorithm}[b]
    \caption{Single update in CMA-ES with Margin for optimization problem $\min_{\x} f(\x)$}
    \label{alg:proposed}
    \begin{algorithmic}[1]
        \STATE \textbf{given } $\m[t] \in \R^N$, $\sig[t] \in \R_{+}$, $\C[t] \in \R^{N \times N}$, $\ps[t] \in \R^N$, $\pc[t] \in \R^N$, and $\A[t] \in \R^{N \times N}$~(diagonal matrix)
        \FOR {$i = 1, \ldots, \lambda$}
            \STATE $\y_i \sim \mathcal{N}(\boldsymbol{0}, \C[t])$
            \STATE $\x_i \leftarrow \m[t] + \sig[t] \y_i$
            \STATE $\boldsymbol{v}_i \leftarrow \m[t] + \sig[t] \A[t] \y_i^\top$
            \STATE $\bar{\boldsymbol{v}}_i \leftarrow \textsc{Encoding}_f(\boldsymbol{v}_i)$
        \ENDFOR
        \STATE Sort $\{\x_{1:\lambda}, \x_{2:\lambda}, \ldots, \x_{\lambda:\lambda}\}$ and $\{\y_{1:\lambda}, \y_{2:\lambda}, \ldots, \y_{\lambda:\lambda}\}$ so that the indices correspond to $f(\bar{\boldsymbol{v}}_{1:\lambda}) \leq f(\bar{\boldsymbol{v}}_{2:\lambda}) \leq \ldots \leq f(\bar{\boldsymbol{v}}_{\lambda:\lambda})$
        \STATE $\m[t+1] \leftarrow \m[t] + c_{m} \sum_{i=1}^{\mu} w_{i} ( \x_{i:\lambda} - \m[t] )$
        \STATE $\ps[t+1] \leftarrow (1-c_\sigma)\ps[t] + \sqrt{c_\sigma(2-c_\sigma)\muw} {\C[t]}^{-\frac12} \sum_{i=1}^\mu \w_i \y_{i:\lambda}$
        \STATE $h_\sigma \leftarrow \mathds{1} {\{\|\ps[t+1]\| < \sqrt{1-(1-c_\sigma)^{2(t+1)}}\left(1.4+\frac{2}{N+1}\right)\E [\|\mathcal{N}(\boldsymbol{0}, \boldsymbol{I}) \|]\}}$
        \STATE $\pc[t+1] \leftarrow (1-c_c) \pc[t] + h_\sigma \sqrt{c_c(2-c_c)\muw} \sum_{i=1}^\mu \w_i \y_{i:\lambda}$
        \STATE $\C[t+1] \leftarrow \left(1 - c_1 - c_\mu \sum_{i=1}^\lambda w_i + (1-h_\sigma)c_1 c_c(2-c_c) \right) \C[t] + \underbrace{c_1 \pc[t+1]{\pc[t+1]}^\top}_{\text{rank-one update}} + \underbrace{c_\mu \sum_{i=1}^\lambda w_i^{\circ} \y_{i:\lambda}\y_{i:\lambda}^\top}_{\text{rank-}\mu\text{ update}}$
        \STATE $\sig[t+1] \leftarrow \sig[t] \exp \left( \frac{c_\sigma}{d_\sigma} \left( \frac{\|\ps[t+1] \|}{\E [\|\mathcal{N}(\boldsymbol{0}, \boldsymbol{I}) \|]} - 1 \right) \right)$
        \STATE $\m[t+1], \boldsymbol{A}^{(t+1)} \leftarrow \textrm{MarginCorrection}\left(\m[t+1], \boldsymbol{A}^{(t)}, \sig[t+1], \C[t+1]\right)$ \label{state:margin_off}
    \end{algorithmic}
\end{algorithm}
\begin{algorithm}[t]
    \caption{Margin Correction}
    \label{alg:margin}
    \begin{algorithmic}[1]
        \STATE \textbf{given } $\boldsymbol{m} \in \R^N$, $\boldsymbol{A} \in \R^{N \times N}$~(diagonal matrix), $\sigma \in \R_{+}$, and $\boldsymbol{C} \in \R^{N \times N}$
        \STATE \texttt{// Margin for Continuous Variables (identity mapping)}
        \FOR {$j = 1, \ldots, N_{\mathrm{co}}$}
            \STATE $[\boldsymbol{m}']_j \leftarrow [\boldsymbol{m}]_j$
            \STATE $\langle \boldsymbol{A}' \rangle_j \leftarrow \langle \boldsymbol{A} \rangle_j$
        \ENDFOR
        \STATE \texttt{// Margin for Binary Variables}
        \FOR {$j = N_{\mathrm{co}}+1, \ldots, N_{\mathrm{co}}+N_{\mathrm{bi}}$}
            \STATE $[\boldsymbol{m}]_j \leftarrow \ell\left([\boldsymbol{m}]_j\right) + \sign\left( [\boldsymbol{m}']_j - \ell\left([\boldsymbol{m}]_j\right) \right) \min \left\{ \left| [\boldsymbol{m}]_j - \ell\left([\boldsymbol{m}]_j\right) \right|, \textrm{CI}_j^{(t+1)} (1-2\alpha) \right\}$
            \STATE $\langle \boldsymbol{A}' \rangle_j \leftarrow \langle \boldsymbol{A} \rangle_j$
        \ENDFOR
        \STATE \texttt{// Margin for Integer Variables}
        \FOR {$j = N_{\mathrm{co}}+N_{\mathrm{bi}}+1, \ldots, N$}
            \IF {$[\boldsymbol{m}]_j \leq \ell_{j,1|2}$ or $\ell_{j,K_j-1|K_j} < [\boldsymbol{m}]_j$}
                \STATE $[\boldsymbol{m}']_j \leftarrow \ell\left([\boldsymbol{m}]_j\right) + \sign\left( [\boldsymbol{m}]_j - \ell\left([\boldsymbol{m}]_j\right) \right) \min \left\{ \left| [\boldsymbol{m}]_j - \ell\left([\boldsymbol{m}]_j\right) \right|, \textrm{CI}_j^{(t+1)} (1-2\alpha) \right\}$
                \STATE $\langle \boldsymbol{A}' \rangle_j \leftarrow \langle \boldsymbol{A} \rangle_j$
            \ELSE
                \STATE $[\boldsymbol{m}']_j \leftarrow \frac{\ell_\textrm{low}\left([\boldsymbol{m}]_j\right) \sqrt{\chi^2_{\textrm{ppf}}(1-2p''_\textrm{up})} + \ell_\textrm{up}\left([\boldsymbol{m}]_j\right) \sqrt{\chi^2_{\textrm{ppf}}(1-2p''_\textrm{low})}}{\sqrt{\chi^2_{\textrm{ppf}}(1-2p''_\textrm{low})} + \sqrt{\chi^2_{\textrm{ppf}}(1-2p''_\textrm{up})}}$
                \STATE $\langle \boldsymbol{A}' \rangle_j \leftarrow \frac{\ell_\textrm{up}\left([\boldsymbol{m}]_j\right) - \ell_\textrm{low}\left([\boldsymbol{m}]_j\right)}{\sigma^{(t+1)} \sqrt{\langle \C[t+1] \rangle_j} \left( \sqrt{\chi^2_{\textrm{ppf}}(1-2p''_\textrm{low})} + \sqrt{\chi^2_{\textrm{ppf}}(1-2p''_\textrm{up})} \right) }$
            \ENDIF
        \ENDFOR
        \RETURN $\boldsymbol{m}'$ and $\boldsymbol{A}'$
    \end{algorithmic}
\end{algorithm}

Finally, the solutions of the simultaneous linear equations for $[\m[t+1]]_j$ and $\langle \A[t] \rangle_j$ are applied to the updated $[\m[t+1]]_j$ and $\langle \A[t+1] \rangle_j$ as follows:
\begin{align}
    [\m[t+1]]_j &\leftarrow \frac{\ell_\textrm{low}\left([\m[t+1]]_j\right) \sqrt{\chi^2_{\textrm{ppf}}(1-2p''_\textrm{up})} + \ell_\textrm{up}\left([\m[t+1]]_j\right) \sqrt{\chi^2_{\textrm{ppf}}(1-2p''_\textrm{low})}}{\sqrt{\chi^2_{\textrm{ppf}}(1-2p''_\textrm{low})} + \sqrt{\chi^2_{\textrm{ppf}}(1-2p''_\textrm{up})}} \\
    \langle \A[t+1] \rangle_j &\leftarrow \frac{\ell_\textrm{up}\left([\m[t+1]]_j\right) - \ell_\textrm{low}\left([\m[t+1]]_j\right)}{ \sigma^{(t+1)} \sqrt{\langle \C[t+1] \rangle_j} \left( \sqrt{\chi^2_{\textrm{ppf}}(1-2p''_\textrm{low})} + \sqrt{\chi^2_{\textrm{ppf}}(1-2p''_\textrm{up})} \right) }
\end{align}
Correcting $\m[t+1]$ and $\A[t]$ in this way bounds both $\Pr\bigl([\boldsymbol{v}]_j \leq  \ell_\textrm{low} \bigr. \bigl. \bigl([\m[t+1]]_j\bigr)\bigr)$ and $\Pr\bigl(\ell_\textrm{up}\bigl([\m[t+1]]_j \bigr) < [\boldsymbol{v}]_j \bigr)$ above $\alpha/2$, as indicated by the solid line in the right of Figure~\ref{fig:correction}. Moreover, we note that there are cases where $p_\textrm{mid}$, $\Pr(0.5 < [\m[t+1]]_1 \leq 1.5)$ in the right of Figure~\ref{fig:correction}, is less than $\alpha/2$ even with the margin. In that case, the variance is sufficiently large that no fixation of the discrete variable occurs in the corresponding dimension. Finally, we show the algorithm details in Algorithm~\ref{alg:proposed}, and the margin correction therein is described in Algorithm~\ref{alg:margin}.

\section{Experiment and Result}
\label{sec:exp_and_res}
We apply the proposed method to the MI-BBO optimization problem for several benchmark functions to validate its robustness and efficiency. In Section \ref{sec:alpha_search}, we check the performance changes of the proposed method according to the hyperparameter $\alpha$. In Section \ref{sec:benchmark}, we check the difference in the search success rate and the number of evaluations between the proposed method and CMA-ES-IM for several artificial MI-BBO benchmark functions. The definitions of the benchmark functions used in this section are listed as below:
\begin{itemize}
    \item  $\textsc{SphereOneMax}(\bar{x}) = \sum_{j=1}^{N_{\mathrm{co}}} \lbrack \bar{\x} \rbrack_j^2 + N_{\mathrm{bi}} - \sum_{k=N_{\mathrm{co}}+1}^N \lbrack \bar{\x} \rbrack_k$
    \item $\textsc{SphereLeadingOnes}(\bar{x}) = \sum_{j=1}^{N_{\mathrm{co}}}\lbrack \bar{\x} \rbrack_j^2 + N_{\mathrm{bi}} - \sum_{k=N_{\mathrm{co}}+1}^N \prod_{l = N_{\mathrm{co}}+1}^k \lbrack \bar{\x} \rbrack_l$
    \item $\textsc{EllipsoidOneMax}(\bar{x}) = \sum_{j=1}^{N_{\mathrm{co}}}\left(1000^{\frac{j-1}{N_{\mathrm{co}}-1}}\lbrack \bar{\x} \rbrack_j \right)^2 + N_{\mathrm{bi}} - \sum_{k=N_{\mathrm{co}}+1}^N \lbrack \bar{\x} \rbrack_k$
    \item $\textsc{EllipsoidLeadingOnes}(\bar{x}) = \sum_{j=1}^{N_{\mathrm{co}}}\left(1000^{\frac{j-1}{N_{\mathrm{co}}-1}}\lbrack \bar{\x} \rbrack_j \right)^2 + N_{\mathrm{bi}} - \sum_{k=N_{\mathrm{co}}+1}^N \prod_{l = N_{\mathrm{co}}+1}^k \lbrack \bar{\x} \rbrack_l$
    \item $\textsc{SphereInt}(\bar{x}) = \sum_{j=1}^{N}\lbrack \bar{\x} \rbrack_j^2$
    \item $\textsc{EllipsoidInt}(\bar{x}) = \sum_{j=1}^{N}\left(1000^{\frac{j-1}{N-1}}\lbrack \bar{\x} \rbrack_j \right)^2$ 
\end{itemize}
In all functions, the first $N_{\mathrm{co}}$ variables are continuous, whereas the last $N - N_{\mathrm{co}}$ variables are binary or integer;
i.e., \textsc{SphereOneMax}, \textsc{SphereLeadingOnes}, \textsc{EllipsoidOneMax}, and \textsc{EllipsoidLeadingOnes} functions include continuous and binary variables.
On the other hand, \textsc{SphereInt} and \textsc{EllipsoidInt} functions include continuous and integer variables.
In all the experiments, we adopted the default parameters of the CMA-ES listed in Table \ref{tb:cma_params}.

\subsection{Hyperparameter Sensitivity for $\alpha$}
\label{sec:alpha_search}
We use the \textsc{SphereInt} function of the objective function and adopt the same initialization for the distribution parameters and termination condition as in Section \ref{sec:preliminary}. The number of dimensions $N$ is set to 20, 40, or 60, and the numbers of the continuous and integer variables are $N_{\mathrm{co}} = N_{\mathrm{int}} = N/2$. The integer variables are assumed to take values in the range $\lbrack -10, 10 \rbrack$. 
We argue that it is reasonable that the hyperparameter $\alpha$, which determines the margin in the proposed method, should depend on the number of dimensions $N$ and the sample size $\lambda$.
\footnote{
Based on our preliminary experiments, we decided to use $N$ dependence instead of $N_{\mathrm{bi}} + N_{\mathrm{int}}$. This dependence becomes particularly important in scenarios where the objective function involves a high proportion of continuous variables and/or is ill-conditioned. In such cases, continuous variables tend to be optimized later than discrete variables. Our preliminary experiments have revealed that when the margin ($\alpha$) is large, the convergence of continuous variables becomes significantly slower, resulting in inefficient optimization. While our $\alpha$ setting has been effective in our experiments, our understanding of the appropriate $\alpha$ setting remains limited, and we defer the detailed investigation of its effects to future work.
}
In this experiment, we evaluate a total of 48 settings except for $\alpha = 1$ which we set as $\alpha = N^{-m} \lambda^{-n} (m, n \in \lbrack 0, 0.5, 1, 1.5, 2, 2.5, 3 \rbrack)$. In each setting, 100 trials are performed independently using different seed values. 

\begin{figure*}[t]
    \centering
    \includegraphics[width=0.95\linewidth]{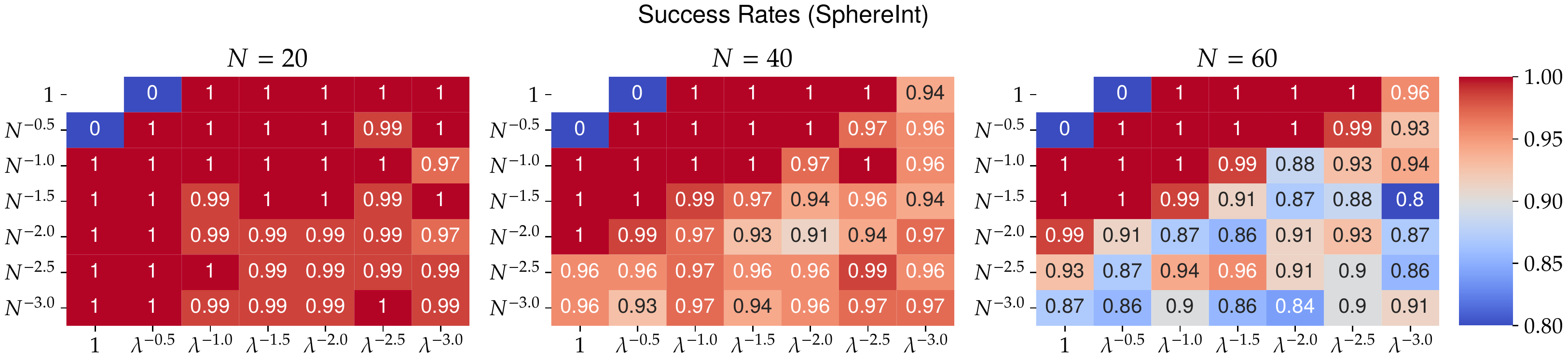}
    \vspace{1.0mm}
    \includegraphics[width=0.95\linewidth]{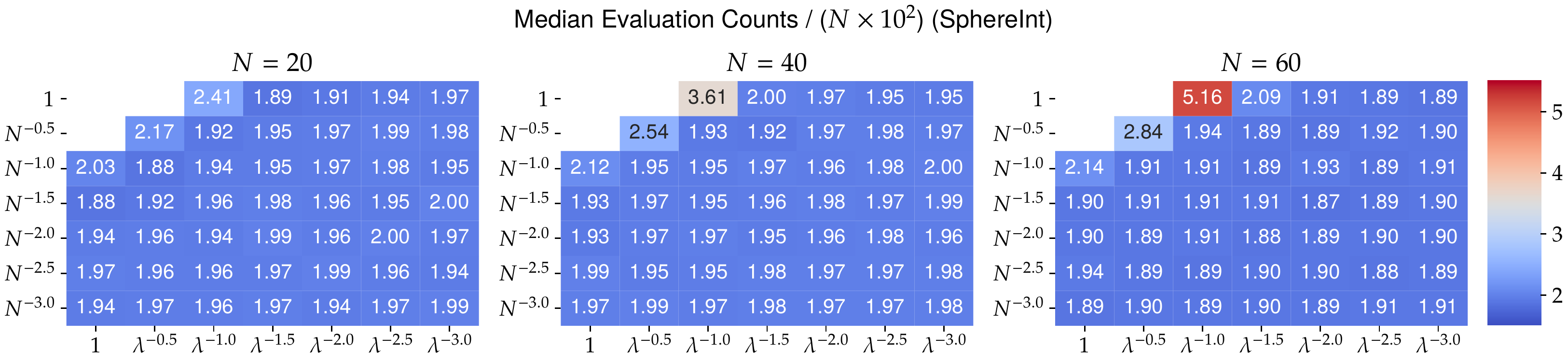}
    \vspace{-2.5mm}
    \caption{Heatmap of the success rate (top) and the median evaluation counts for successful cases (bottom) in the $N$-dimensional \textsc{SphereInt} function when the hyperparameter $\alpha=N^{-m}\lambda^{-n}$ of the proposed method is changed.}
    \label{fig:grid_seaech_success}
\end{figure*}

\begin{table*}[tbph]
  \caption{Comparison of evaluation counts and success rate for the benchmark functions. The evaluation counts are reported as the median value of successful trials. The bold fonts represent the best median evaluation counts and the best success rates among the methods. The inside of the parentheses represents the interquartile range (IQR).}
  \vspace{-2.5mm}
  \label{table:result_benchmark_funcion}
  \centering
  \begin{tabular}{c|c|cc|cc|cc}
    \toprule
    \multirow{3}{*}{Function} & \multirow{3}{*}{$N$} & \multicolumn{2}{c}{CMA-ES-IM} & \multicolumn{2}{|c|}{CMA-ES-IM \& Box-constraint} & \multicolumn{2}{c}{CMA-ES w. Margin (Proposed)} \\
    & & Evaluation & Success & Evaluation & Success & Evaluation  & Success \\
    & & Counts & Rates & Counts & Rates & Counts & Rates \\
    \midrule
    \multirow{3}{*}{\textsc{SphereOneMax}}
        & 20 & \textbf{2964} (174) & 83/100 & 4752 (100) & \textbf{100/100} & 3876 (435) & \textbf{100/100} \\
        & 40 & \textbf{5745} (558) & 62/100 & 7575 (543) & 64/100 & 7995 (514) & \textbf{100/100} \\
        & 60 & \textbf{8112} (340) & 46/100 & 12240 (336) & 21/100 & 12408 (1012) & \textbf{100/100} \\
    \midrule
    \multirow{3}{*}{\textsc{SphereLeadingOnes}} 
        & 20 & \textbf{2904} (192) & 77/100 & 5124 (296) & \textbf{100/100} & 4158 (339) & \textbf{100/100} \\
        & 40 & \textbf{5647} (296) & 24/100 & 8280 (112) & 10/100 & 8505 (724) & \textbf{100/100} \\
        & 60 & \textbf{8816} (12) & 8/100 & 12624 (96) & 2/100 & 13424 (1008) & \textbf{100/100} \\
    \midrule
    \multirow{3}{*}{\textsc{EllipsoidOneMax}}
        & 20 & \textbf{9918} (396) & 20/100 & 26700 (7020) & 96/100 & 11172 (666) & \textbf{100/100} \\
        & 40 & \textbf{35325} (0) & 1/100 & 79912 (28661) & 14/100 & 40590 (1789) & \textbf{100/100} \\
        & 60 & \textbf{78560} (0) & 5/100 & 283440 (0) & 1/100 & 88064 (3536) & \textbf{100/100} \\
    \midrule
    \multirow{3}{*}{\textsc{EllipsoidLeadingOnes}} 
        & 20 & \textbf{10104} (441) & 14/100 & 24880 (7305) & 92/100 & 11454 (876) & \textbf{100/100} \\
        & 40 & - & 0/100 & 85807 (13335) & 4/100 & \textbf{41048} (1744) & \textbf{100/100} \\
        & 60 & - & 0/100 & - & 0/100 & \textbf{91496} (3488) & \textbf{100/100} \\
    \midrule
    \multirow{3}{*}{\textsc{SphereInt}}
        & 20 & 5130 (477) & 86/100 & 5280 (744) & 89/100 & \textbf{3840} (306) & \textbf{100/100} \\
        & 40 & 7950 (697) & 71/100 & 8070 (555) & 85/100 & \textbf{7838} (458) & \textbf{100/100} \\
        & 60 & 13184 (816) & 41/100 & 12992 (544) & 34/100 & \textbf{11512} (544) & \textbf{100/100} \\
    \midrule
    \multirow{3}{*}{\textsc{EllipsoidInt}} 
        & 20 & 19128 (3192) & 76/100 & 19476 (5028)  & 73/100 & \textbf{8418} (837) & \textbf{100/100} \\
        & 40 & 43935 (4095) & 73/100 & 43440 (4762) & 83/100 & \textbf{22815} (1733) & \textbf{100/100} \\
        & 60 & 89200 (6152) & 54/100 & 86848 (6136) & 54/100 & \textbf{42000} (3320) & \textbf{100/100} \\
    \bottomrule
  \end{tabular}
\end{table*}

\paragraph{Results and Discussion}
Figure~\ref{fig:grid_seaech_success} shows the success rate and the median evaluation count for successful cases in each setting. For the success rate, when $\alpha$ is set to a large value as $N^{-0.5}$ or $\lambda^{-0.5}$, all trials fail. Additionally, when $\alpha$ is set as smaller than $(N\lambda)^{-1}$, the success rate tends to decrease as the number of dimensions $N$ increases. If $\alpha$ is too large, the probability of the integer changing is also too large and the optimization is unstable;
however, if $\alpha$ is too small, it is difficult to get out of the stagnation because the conditions under which the mean $[\boldsymbol{m}]_{j}$ and affine matrix $\boldsymbol{A}$ corrections are applied become more stringent. 
For the median evaluation count, there is no significant difference for any dimension except for
$\alpha \in \{ N^{-1}, (N \lambda)^{-0.5}, \lambda^{-1} \}$. 
Therefore, for robustness and efficiency reasons, we use $\alpha = (N\lambda)^{-1}$ as a default parameter in the subsequent experiments in this study.

\subsection{Comparison of Optimization Performance}
\label{sec:benchmark}
We compare the optimization performance on the benchmark functions listed in Section \ref{sec:exp_and_res} for the proposed method, CMA-ES-IM, and CMA-ES-IM with box constraints to evaluate the effectiveness of the proposed integer handling.
The detailed comparison of the proposed method with other mixed-integer black-box optimization methods will be addressed in future work.
As in Section \ref{sec:alpha_search}, the number of dimensions $N$ is set to 20, 40, and 60. The number of continuous and integer variables are $N_{\mathrm{co}} = N_{\mathrm{bi}} = N_{\mathrm{int}} = N/2$, respectively. For CMA-ES-IM with the box constraint, \textsc{SphereOneMax}, \textsc{SphereLeadingOnes}, \textsc{EllipsoidOneMax}, and \textsc{ElipsoidLeadingOnes} functions are given the constraint $\lbrack \x \rbrack_{j} \in \lbrack -1, 1 \rbrack$ corresponding to the binary variables, and other functions are given the constraint $\lbrack \x \rbrack_{j} \in \lbrack -10, 10 \rbrack$ corresponding to the integer variables. The calculating method of the penalty for violating the constraints is the same as in Section \ref{sec:preliminary}. The optimization is successful when the best-evaluated value is less than $10^{-10}$, whereas the optimization is stopped when the minimum eigenvalue of $\sigma^{2} \boldsymbol{C}$ is less than $10^{-30}$ or the condition number of $\boldsymbol{C}$ exceeds $10^{14}$.

\paragraph{Results and Discussion}
Table \ref{table:result_benchmark_funcion} summarizes the median evaluation counts and success rates in each setting. Comparing the proposed method with the CMA-ES-IM with and without the box constraint for the \textsc{SphereOneMax} and \textsc{SphereLeadingOnes} functions, the CMA-ES-IM reaches the optimal solution in fewer evaluation counts. However, the success rate of CMA-ES-IM with and without the box constraint decreases as the number of dimensions increases. However, the success rate of the proposed method remains 100\% regardless of the increase in the number of dimensions. 
For the \textsc{EllipsoidOneMax} and \textsc{EllipsoidLeadingOnes} functions, the CMA-ES-IM without the box constraint fails on most of the trials in all dimensions, and the CMA-ES-IM with box constraint has a relatively high success rate in $N = 20$ but deteriorates rapidly in $N = 40$ or more dimensions.
In contrast, the proposed method maintains a 100\% success rate and reaches the optimal solution in fewer evaluation counts than the CMA-ES-IM with the box constraint in $N = 20$ or $N = 40$. For the \textsc{SphereInt} and \textsc{EllipsoidInt} functions, the proposed method successfully optimizes with fewer evaluations in all dimensions than the other methods, maintaining a 100\% success rate. These results show that the proposed method can perform the MI-BBO robustly and efficiently for multiple functions with an increasing number of dimensions.

\section{Extension to Multi-Objective Optimization}
\label{sec:multiobjective}
In this section, we present the application of the margin correction to multi-objective mixed-integer black-box optimization~(MO-MI-BBO). First, we briefly review the MO-MI-BBO and the MO-CMA-ES, followed by a discussion of the introduction of the margin into the MO-CMA-ES.

\subsection{Multi-Objective Mixed-Integer Optimization}
A multi-objective optimization problem is an optimization problem that targets multiple objective functions. In this paper, we will consider multi-objective optimization involving continuous and integer variables and refer to it as multi-objective mixed-integer optimization~(MO-MI-BBO).
The $M$-objective MO-MI-BBO is formally defined as follows:
\begin{align*}
    \underset{\x}{\text{minimize}} \enspace &\left( f_1(\x), \ldots, f_M(\x) \right) \\
    \text{where} \enspace &f_m:\R^{N_{\mathrm{co}}} \times \{ 0,1 \}^{N_{\mathrm{bi}}} \times \Z^{N_{\mathrm{in}}} \rightarrow \R \\
    &N = N_{\mathrm{co}} + N_{\mathrm{bi}} + N_{\mathrm{in}} \\
    &\hspace{-1.5mm}\begin{array}{ll}
        x_i \in \{0,1\} & \text{if} \enspace N_{\mathrm{co}}+1 \leq i \leq N_{\mathrm{co}}+N_{\mathrm{bi}} \\
        x_i \in [x_i^{\min}, x_i^{\max}] \cap \Z & \text{if} \enspace N_{\mathrm{co}}+N_{\mathrm{bi}}+1 \leq i \leq N
    \end{array}
\end{align*}
The goal of multi-objective optimization, including MO-MI-BBO, is to find a diverse set of Pareto-optimal solutions. The set of Pareto-optimal solutions is often evaluated by the hypervolume measure or $\mathcal{S}$-metric, which is defined as the Lebesgue measure of the union of hypercuboids in the objective space~\cite{HV:1998}. Multi-objective evolutionary algorithms often use this hypervolume measure and non-dominated sorting, as described below during the search.

\subsection{MO-CMA-ES}
The multi-objective covariance matrix adaptation strategy~(MO-CMA-ES)~\cite{igel2007covariance} is an extension of the CMA-ES for multi-objective black-box optimization. It combines the strategy parameter adaptation of (1+1)-CMA-ES with multi-objective selection based on non-dominated sorting~\cite{nondominated} and the contributing hypervolume~\cite{contributingHV}.
Moreover, the success rule is applied to each pair of parent and offspring, and the mutation is regarded as successful if the offspring has a higher rank than the parent.
In contrast, \citet{voss2010improved} proposed a new update scheme for the MO-CMA-ES, which considers a mutation successful if the offspring are selected for the next parent population. This significantly improved performance in the MO-CMA-ES by avoiding premature convergence of the step size. In this paper, we focus on this improved MO-CMA-ES as a baseline. For simplicity, we consider the case $\mu = \lambda$, where the number of parent and offspring individuals is equal. When applying the improved MO-CMA-ES to MO-MI-BBO, we repeat the following steps until a termination criterion is satisfied.

\paragraph{Notation}
The improved MO-CMA-ES denotes the $i$-th individual $\pa[t]_i$ in generation $t$ by the tuple $\left[ \x^{(t)}_i, {\bar{\boldsymbol{v}}}_i^{(t)},\right. \linebreak \left. \bar{p}_{\suc, i}^{(t)}, {\sigma}_i^{(t)}, \boldsymbol{p}_{c, i}^{(t)}, \boldsymbol{C}_{i}^{(t)} \right]$, where $\x^{(t)}_i$ is the search point, ${\bar{\boldsymbol{v}}}_i^{(t)}$ is the discretized search point, $\bar{p}_{\suc, i}^{(t)}$ is the success probability, ${\sigma}_i^{(t)}$ is the global step size, $\boldsymbol{p}_{c, i}^{(t)}$ is the evolution path, $\boldsymbol{C}_{i}^{(t)}$ is the covariance matrix of the search distribution.

\paragraph{Sample Offspring individuals}
In the $t$-th iteration, each of the $\lambda$ offspring individuals $\of[t+1]_i$~($i = 1, \ldots, \lambda$) is copied from its parent $\pa[t]_i$. Then mutate ${\x'}^{(t+1)}_i$ according to ${\sigma}_i^{(t)}$ and $\C[t]_i$ as follows.
\begin{align}
    {\x'}^{(t+1)}_i \sim {\x}^{(t)}_i + {\sigma}_i^{(t)} \mathcal{N}(\boldsymbol{0}, \C[t]_i)
\end{align}
Moreover, the discretized search point ${\bar{\boldsymbol{v}'}}_i^{(t+1)}$ is calculated using $\textsc{Encoding}_f$ introduced in Section~\ref{ssec:encoding_function}.
\begin{align}
    {\bar{\boldsymbol{v}'}}_i^{(t+1)} = \textsc{Encoding}_f({\x'}^{(t+1)}_i)
\end{align}

\paragraph{Ranking the individuals}
Rank the individuals in the set of parent and offspring individuals $Q^{(t)}$ based on the non-dominated sorting and the contributing hypervolume. For details, see~\cite[Equation~4]{voss2010improved}. However, note that in the case of MO-MI-BBO, the argument of the objective function is not $x_i^{(t)}$ but ${\bar{\boldsymbol{v}}}_i^{(t)}$, i.e., $f_m(\pa[t]_i)$ implies $f_m({\bar{\boldsymbol{v}}}_i^{(t)})$, not $f_m(\boldsymbol{x}_i^{(t)})$. The $i$-th best individual ranked by non-dominated sorting and the contributing hypervolume in $Q^{(t)}$ is defined as $Q^{(t)}_{\prec : i}$. Then, the next parent population is denoted by $Q^{(t+1)} = \left\{ Q^{(t)}_{\prec : i} \middle| 1 \leq i \leq \lambda \right\}$ since the top $\lambda$ individuals are selected from $Q^{(t)}$.

\paragraph{Success Indicator}
In the MO-CMA-ES, the step size and covariance matrix are updated based on the success rule. The success of the mutation from parent individual $\pa[t]_i$ to offspring individual $\of[t+1]_i$ is determined by whether the offspring individual $\of[t+1]_i$ is selected for the next parent population $Q^{(t+1)}$, i.e., the success indicator $\suc_{Q^{(t)}} \left( \pa[t]_i, \of[t+1]_i \right)$ is defined by
\begin{align}
    \suc_{Q^{(t)}} \left( \pa[t]_i, \of[t+1]_i \right) =
    \begin{cases}
        1 \enspace &\text{if} \enspace \of[t+1]_i \in Q^{(t+1)}\\
        0 \enspace &\text{otherwise}
    \end{cases}
    \enspace.
\end{align}

\paragraph{Compute Success Probabilities and Step-Size Adaptation}
The smoothed success probabilities are updated as
\begin{align}
    \psuof{i}{t+1} \gets (1-c_p)\psuof{i}{t+1} + c_p \suc_{Q^{(t)}} \left( \pa[t]_i, \of[t+1]_i \right) \enspace,
\end{align}
where $c_p$ is the success rate averaging parameter. According to each success probability, step-size adaptation is performed as
\begin{align}
    {\sigma'}_i^{(t+1)} \gets {\sigma'}_i^{(t+1)} \exp \left( \frac{1}{d} \frac{\psuof{i}{t+1} - p_{\suc}^{\mathrm{target}}}{1 - p_{\suc}^{\mathrm{target}}} \right) \enspace,
\end{align}
where $p_{\suc}^{\mathrm{target}}$ is a hyperparameter that corresponds to the target success probability.

\paragraph{Update Covariance Matrix and Evolution Path}
Based on the success probability $\psuof{i}{t+1}$, the evolution path $\pcof{i}{t+1}$ and covariance matrix $\Cof{i}{t+1}$ are updated. If $\psuof{i}{t+1} < p_{\mathrm{thresh}}$,
\begin{align}
    \pcof{i}{t+1} &\gets (1-c_c) \pcof{i}{t+1} + \sqrt{c_c(2-c_c)} \frac{{\x'}^{(t+1)}_i - \x^{(t)}_i}{{\sigma}_i^{(t)}} \enspace, \\
    \Cof{i}{t+1} &\gets (1-c_{\mathrm{cov}})\Cof{i}{t+1} + c_{\mathrm{cov}} \pcof{i}{t+1} {\pcof{i}{t+1}}^\top \enspace.
\end{align}
If $\psuof{i}{t+1} \geq p_{\mathrm{thresh}}$,
\begin{align}
    \pcof{i}{t+1} &\gets (1-c_c) \pcof{i}{t+1} \enspace, \\
    \Cof{i}{t+1} &\gets (1-c_{\mathrm{cov}})\Cof{i}{t+1} + c_{\mathrm{cov}} \left( \pcof{i}{t+1} {\pcof{i}{t+1}}^\top + c_c (2-c_c) \Cof{i}{t+1} \right) \enspace.
\end{align}

\paragraph{Update Parent Individuals}
The success probability $\bar{p}_{\suc, i}^{(t)}$ and step size ${\sigma}_i^{(t)}$ of the parent individuals $\pa[t]_i$ are adapted as follows.
\begin{align}
    \bar{p}_{\suc, i}^{(t)} &\gets (1-c_p) \bar{p}_{\suc, i}^{(t)} + c_p \suc_{Q^{(t)}} \left( \pa[t]_i, \of[t+1]_i \right) \enspace, \\
    {\sigma}_i^{(t)} &\gets {\sigma}_i^{(t)} \exp \left( \frac{1}{d} \frac{\bar{p}_{\suc, i}^{(t)} - p_{\suc}^{\mathrm{target}}}{1 - p_{\suc}^{\mathrm{target}}} \right) \enspace.
\end{align}

\paragraph{Select New Parent Population}
Finally, the new parent population is selected from the set of parent and offspring individuals. The top $\lambda$ individuals are the parents of the next generation.
\begin{align}
    Q^{(t+1)} = \left\{ Q^{(t)}_{\prec : i} \middle| 1 \leq i \leq \lambda \right\}
\end{align}

\subsection{Introduction of Margin to MO-CMA-ES}
To address the fixation that can occur when the MO-CMA-ES is applied to MO-MI-BBO, we introduce the margin to MO-CMA-ES. In introducing the margin, there are two differences between the MO-CMA-ES and the single-objective CMA-ES introduced in Section~\ref{sec:cma-es}. The first is that in the MO-CMA-ES, the step size and covariance are updated for each individual. Therefore, the matrix $\boldsymbol{A}$ for margin correction should also be prepared for each individual, i.e., the tuple for the individual $\pa[t]_i$ is redefined as $\left[ \x^{(t)}_i, {\bar{\boldsymbol{v}}}_i^{(t)},  \bar{p}_{\suc, i}^{(t)}, {\sigma}_i^{(t)}, \boldsymbol{p}_{c, i}^{(t)}, \boldsymbol{C}_{i}^{(t)}, \boldsymbol{A}_{i}^{(t)} \right]$.
The second is that the MO-CMA-ES uses an elitist evolution strategy, and the margin correction targets the search point $\x^{(t)}_i$. To discuss the validity of the margin correction in the elitist evolution strategy, we show that the discretized search points are invariant before and after the margin correction in the following proposition.

\begin{prop}~\label{prop:after_margin_correction}
Consider the margin correction for an individual $\pa[t]_i$. Let $\tilde{\boldsymbol{x}}_i^{(t)}$ be the margin-corrected search point. Then, it holds
\begin{align}
    \textsc{Encoding}_f (\boldsymbol{x}_i^{(t)}) = \textsc{Encoding}_f (\tilde{\boldsymbol{x}}_i^{(t)})
\end{align}
\end{prop}

\begin{proof}
We will show that it holds, for $j=1, \ldots, N$,
\begin{align}
    \textsc{Encoding}_f ([\boldsymbol{x}_i^{(t)}]_j) = \textsc{Encoding}_f ([\tilde{\boldsymbol{x}}_i^{(t)}]_j) \enspace. \label{eq:encoded_jth}
\end{align}
First, when $j = 1, \ldots, N_{\text{co}}$, it is obvious that \eqref{eq:encoded_jth} holds since $[\boldsymbol{x}_i^{(t)}]_j = [\tilde{\boldsymbol{x}}_i^{(t)}]_j$. Next, when $j = N_{\text{co}}+1, \ldots, N_{\text{co}}+N_{\text{bi}}$, 
\begin{align}
    [\tilde{\boldsymbol{x}}_i^{(t)}]_j = \ell\left([\boldsymbol{x}_i^{(t)}]_j\right) + \sign\left( [\boldsymbol{x}_i^{(t)}]_j - \ell\left([\boldsymbol{x}_i^{(t)}]_j\right) \right) \min \left\{ \left| [\boldsymbol{x}_i^{(t)}]_j - \ell\left([\boldsymbol{x}_i^{(t)}]_j\right) \right|, \textrm{CI}_j^{(t)} (1-2\alpha) \right\} \enspace.
\end{align}
If $\ell\left([\boldsymbol{x}_i^{(t)}]_j\right) < [\boldsymbol{x}_i^{(t)}]_j$,
\begin{align}
    [\tilde{\boldsymbol{x}}_i^{(t)}]_j = \ell\left([\boldsymbol{x}_i^{(t)}]_j\right) + \min \left\{ \left| [\boldsymbol{x}_i^{(t)}]_j - \ell\left([\boldsymbol{x}_i^{(t)}]_j\right) \right|, \textrm{CI}_j^{(t)} (1-2\alpha) \right\} > \ell\left([\boldsymbol{x}_i^{(t)}]_j\right) \enspace.
\end{align}
If $[\boldsymbol{x}_i^{(t)}]_j \leq \ell\left([\boldsymbol{x}_i^{(t)}]_j\right)$,
\begin{align}
    [\tilde{\boldsymbol{x}}_i^{(t)}]_j = \ell\left([\boldsymbol{x}_i^{(t)}]_j\right) - \min \left\{ \left| [\boldsymbol{x}_i^{(t)}]_j - \ell\left([\boldsymbol{x}_i^{(t)}]_j\right) \right|, \textrm{CI}_j^{(t)} (1-2\alpha) \right\} \leq \ell\left([\boldsymbol{x}_i^{(t)}]_j\right) \enspace.
\end{align}
Therefore, we have
\begin{align}
    \begin{cases}
        \ell\left([\boldsymbol{x}_i^{(t)}]_j\right) < [\tilde{\boldsymbol{x}}_i^{(t)}]_j & \text{if} \enspace \ell\left([\boldsymbol{x}_i^{(t)}]_j\right) < [\boldsymbol{x}_i^{(t)}]_j \\
        [\tilde{\boldsymbol{x}}_i^{(t)}]_j \leq \ell\left([\boldsymbol{x}_i^{(t)}]_j\right) & \text{if} \enspace [\boldsymbol{x}_i^{(t)}]_j \leq \ell\left([\boldsymbol{x}_i^{(t)}]_j\right)
    \end{cases} \enspace,
\end{align}
which shows that \eqref{eq:encoded_jth} holds when $j = N_{\text{co}}+1, \ldots, N_{\text{co}}+N_{\text{bi}}$. Finally, when $j = N_{\text{co}}+N_{\text{bi}}+1, \ldots, N$,
\begin{align}
    [\tilde{\boldsymbol{x}}_i^{(t)}]_j &= \frac{\ell_\textrm{low}\left([\boldsymbol{x}_i^{(t)}]_j\right) \sqrt{\chi^2_{\textrm{ppf}}(1-2p''_\textrm{up})} + \ell_\textrm{up}\left([\boldsymbol{x}_i^{(t)}]_j\right) \sqrt{\chi^2_{\textrm{ppf}}(1-2p''_\textrm{low})}}{\sqrt{\chi^2_{\textrm{ppf}}(1-2p''_\textrm{low})} + \sqrt{\chi^2_{\textrm{ppf}}(1-2p''_\textrm{up})}} \\
    &= \ell_\textrm{low}\left([\boldsymbol{x}_i^{(t)}]_j\right) + \left(\ell_\textrm{up}\left([\boldsymbol{x}_i^{(t)}]_j\right) - \ell_\textrm{low}\left([\boldsymbol{x}_i^{(t)}]_j\right)\right) \frac{\sqrt{\chi^2_{\textrm{ppf}}(1-2p''_\textrm{up})}}{\sqrt{\chi^2_{\textrm{ppf}}(1-2p''_\textrm{up})} + \sqrt{\chi^2_{\textrm{ppf}}(1-2p''_\textrm{low})}} \enspace.
\end{align}
Since
\begin{align}
    0 < \frac{\sqrt{\chi^2_{\textrm{ppf}}(1-2p''_\textrm{up})}}{\sqrt{\chi^2_{\textrm{ppf}}(1-2p''_\textrm{up})} + \sqrt{\chi^2_{\textrm{ppf}}(1-2p''_\textrm{low})}} < 1 \enspace,
\end{align}
we have
\begin{align}
    \ell_\textrm{low}\left([\boldsymbol{x}_i^{(t)}]_j\right) < [\tilde{\boldsymbol{x}}_i^{(t)}]_j < \ell_\textrm{up}\left([\boldsymbol{x}_i^{(t)}]_j\right) \enspace,
\end{align}
which shows that \eqref{eq:encoded_jth} holds when $j = N_{\text{co}}+N_{\text{bi}}+1, \ldots, N$. From the above, \eqref{eq:encoded_jth} holds when $j = 1, \ldots, N$, and $\textsc{Encoding}_f (\boldsymbol{x}_i^{(t)}) = \textsc{Encoding}_f (\tilde{\boldsymbol{x}}_i^{(t)})$ holds. This is the end of the proof.
\end{proof}

\begin{algorithm}[tbh]
    \caption{Single update in MO-CMA-ES with Margin}
    \label{alg:biobj-proposed}
    \begin{algorithmic}[1]
        \STATE \textbf{given} $Q^{(t)}$
        \FOR {$i = 1, \ldots, \lambda$}
            \STATE $\of[t+1]_i \leftarrow \pa[t]_i$
            \STATE $\y_i \sim \mathcal{N}(\boldsymbol{0}, \C[t]_i)$
            \STATE ${\x'}^{(t+1)}_i \leftarrow \x^{(t)}_i + {\sigma}_i^{(t)} \y_i$
            \begin{shaded}
            \STATE ${\boldsymbol{v}'}_i^{(t+1)} \leftarrow \x^{(t)}_i + {\sigma}_i^{(t)} \A[t]_i \y_i^\top$ \label{state:affine}
            \STATE ${\bar{\boldsymbol{v}'}}_i^{(t+1)} \leftarrow \textsc{Encoding}_f({\boldsymbol{v}'}_i^{(t+1)})$ \label{state:encoding}
            \end{shaded}
            \STATE $Q^{(t)} \leftarrow Q^{(t)} \cup \left\{ \of[t+1]_i \right\}$
        \ENDFOR
        \FOR {$i = 1, \ldots, \lambda$}
            \STATE $\psuof{i}{t+1} \leftarrow (1-c_p)\psuof{i}{t+1} + c_p \suc_{Q^{(t)}} \left( \pa[t]_i, \of[t+1]_i \right)$
            \STATE ${\sigma'}_i^{(t+1)} \leftarrow {\sigma'}_i^{(t+1)} \exp \left( \frac{1}{d} \frac{\psuof{i}{t+1} - p_{\suc}^{\mathrm{target}}}{1 - p_{\suc}^{\mathrm{target}}} \right)$
            \IF {$\psuof{i}{t+1} < p_{\mathrm{thresh}}$}
                \STATE $\pcof{i}{t+1} \leftarrow (1-c_c) \pcof{i}{t+1} + \sqrt{c_c(2-c_c)} \frac{{\x'}^{(t+1)}_i - \x^{(t)}_i}{{\sigma}_i^{(t)}}$
                \STATE $\Cof{i}{t+1} \leftarrow (1-c_{\mathrm{cov}})\Cof{i}{t+1} + c_{\mathrm{cov}} \pcof{i}{t+1} {\pcof{i}{t+1}}^\top $
            \ELSE
                \STATE $\pcof{i}{t+1} \leftarrow (1-c_c) \pcof{i}{t+1}$
                \STATE $\Cof{i}{t+1} \leftarrow (1-c_{\mathrm{cov}})\Cof{i}{t+1} + c_{\mathrm{cov}} \left( \pcof{i}{t+1} {\pcof{i}{t+1}}^\top + c_c (2-c_c) \Cof{i}{t+1} \right)$
            \ENDIF
            \STATE $\bar{p}_{\suc, i}^{(t)} \leftarrow (1-c_p) \bar{p}_{\suc, i}^{(t)} + c_p \suc_{Q^{(t)}} \left( \pa[t]_i, \of[t+1]_i \right)$
            \STATE ${\sigma}_i^{(t)} \leftarrow {\sigma}_i^{(t)} \exp \left( \frac{1}{d} \frac{\bar{p}_{\suc, i}^{(t)} - p_{\suc}^{\mathrm{target}}}{1 - p_{\suc}^{\mathrm{target}}} \right)$
            \begin{shaded}
            \STATE ${\x'}^{(t+1)}_i, {\boldsymbol{A}'}^{(t+1)}_i \leftarrow \textrm{MarginCorrection}\left({\x'}^{(t+1)}_i, {\boldsymbol{A}'}^{(t+1)}_i, {\sigma'}_i^{(t+1)}, \Cof{i}{t+1}\right)$ \label{state:margin_off}
            \STATE ${\x}^{(t)}_i, {\boldsymbol{A}}^{(t)}_i \leftarrow \textrm{MarginCorrection}\left( {\x}^{(t)}_i, {\boldsymbol{A}}^{(t)}_i, {\sigma}_i^{(t)}, {\boldsymbol{C}'}^{(t)}_i \right)$ \label{state:margin_par}
            \end{shaded}
        \ENDFOR
        \STATE $Q^{(t+1)} \leftarrow \left\{ Q^{(t)}_{\prec : i} \middle| 1 \leq i \leq \lambda \right\}$
    \end{algorithmic}
\end{algorithm}

Proposition~\ref{prop:after_margin_correction} shows that before and after the margin correction, the invariance of the discretized search points is preserved, and the evaluation values of each individual do not change. The ranking and contributing hypervolume of each individual also remain unchanged.

The single update in the MO-CMA-ES with margin is shown in Algorithm~\ref{alg:biobj-proposed}. See Algorithm~\ref{alg:margin} for the margin correction in lines \ref{state:margin_off} and \ref{state:margin_par}. \colorbox[rgb]{0.8, 0.8, 0.8}{Shaded areas} are modifications of the original improved MO-CMA-ES. If lines \ref{state:affine} and \ref{state:encoding} are replaced by ${\bar{\boldsymbol{v}'}}_i^{(t+1)} \leftarrow \textsc{Encoding}_f({\boldsymbol{x}'}_i^{(t+1)})$ and lines \ref{state:margin_off} and \ref{state:margin_par} are removed, the original improved MO-CMA-ES is obtained.

\section{Experiment on Multi-Objective Optimization Problems: Hyperparameter Sensitivity for $\alpha$}
\label{sec:exp_multiobjective}

In this section, we apply MO-CMA-ES with Margin to several multi-objective MI-BBO benchmark functions and validate its robustness and efficiency as same as Section~\ref{sec:exp_and_res}. 
To confirm the effect of introducing the margin, we compare the MO-CMA-ES with Margin with the (vanilla) MO-CMA-ES.
As a benchmark function involving continuous and binary variables, we consider DSLOTZ, which combines \textsc{DoubleSphere} and \textsc{LeadingOnesTailingZeros}~(LOTZ). The \textsc{DoubleSphere} function is a benchmark function to minimize two \textsc{Sphere} functions $f_1(x) = \sum_{i=1}^N x_i^2$ and $f_2(x) = \sum_{i=1}^N (1-x_i)^2$. The Pareto set of \textsc{DoubleSphere} can be expressed analytically; the line segment between the points that minimize $f_1$ and $f_2$, respectively. The LOTZ function is a benchmark function to maximize two functions $f_1(x) = \sum_{i=1}^N \prod_{j=1}^i x_j$ and $f_2(x) = \sum_{i=1}^N \prod_{j=i}^n (1-x_j)$. The first objective $f_1$ is the LeadingOnes function, which counts the number of leading ones of the bit string. On the other hand, the second objective $f_2$ counts the number of tailing zeros. The Pareto set of LOTZ consists of the bit strings ``$11...11$'',  ``$11...10$'', $\ldots$, ``$10...00$'', and ``$00...00$''.

Additionally, as a benchmark function involving continuous and integer variables, we consider \textsc{DSInt}, in which some variables of \textsc{DoubleSphere} are integerized. However, when the \textsc{DoubleSphere} function that minimizes $f_1(x) = \sum_{i=1}^N x_i^2$ and $f_2(x) = \sum_{i=1}^N (1-x_i)^2$ partially is integerized, the integer variable part of of the Pareto set can only take $0^{N_\textrm{in}}$ or $1^{N_\textrm{in}}$, making the optimization of the integer variable part too easy. Therefore, we verify the \textsc{DoubleSphere} that minimizes $f_1(x) = \sum_{i=1}^N x_i^2$ and $f_2(x) = \sum_{i=1}^N (10-x_i)^2$ with a partially discretized function.

\newpage
The definitions of the benchmark functions in this section are listed as below:
\begin{itemize}
    \item \textsc{DSLOTZ} (\textsc{DoubleSphereLeadingOnesTailingZeros})
    \begin{align}
        f_1(\bar{\x}) &= \frac{1}{N_\textrm{co}} \left( \sum_{j=1}^{N_\textrm{co}} [\bar{\x}]_j^2 \right) + \frac{1}{N_\textrm{bi}} \left( N_\textrm{bi} - \sum_{k=N_\textrm{co}+1}^{N} \prod_{l=N_\textrm{co}+1}^k [\bar{\x}]_l \right) \\
        f_2(\bar{\x}) &= \frac{1}{N_\textrm{co}} \left( \sum_{j=1}^{N_\textrm{co}} (1-[\bar{\x}]_j)^2 \right) + \frac{1}{N_\textrm{bi}} \left( N_\textrm{bi} - \sum_{k=N_\textrm{co}+1}^{N} \prod_{l=k}^N (1-[\bar{\x}]_l) \right)
    \end{align}

    \item \textsc{DSInt} (\textsc{DoubleSphereInt})
    \begin{align}
        f_1(\bar{\x}) &= \frac{1}{N_\textrm{co}} \left( \frac{1}{10^2} \sum_{j=1}^{N_\textrm{co}} [\bar{\x}]_j^2 \right) + \frac{1}{N_\textrm{in}} \left( \frac{1}{10^2} \sum_{k=N_\textrm{co}+1}^{N} [\bar{\x}]_k^2 \right) \\
        f_2(\bar{\x}) &= \frac{1}{N_\textrm{co}} \left( \frac{1}{10^2} \sum_{j=1}^{N_\textrm{co}} (10-[\bar{\x}]_j)^2 \right) + \frac{1}{N_\textrm{in}} \left( \frac{1}{10^2} \sum_{k=N_\textrm{co}+1}^{N} (10-[\bar{\x}]_k)^2 \right)
    \end{align}
\end{itemize}
In all functions, the first $N_{\mathrm{co}}$ variables are continuous wheres the last $N - N_{\mathrm{co}}$ variables are binary or integer. In all experiments, we adopt the default parameters of improved MO-CMA-ES as given in \cite{voss2010improved}, and set a reference point as $[5, 5]$.

The step-size and covariance matrix are initialized as $\sigma_{i}^{(0)} = 1$ and $\C[0]_i = \boldsymbol{I}$ for the \textsc{DSLOTZ} function and $\sigma_{i}^{(0)} = 5$ and $\C[0]_i = \boldsymbol{I}$ for the \textsc{DSInt} function.
The initial search point $\x^{(0)}_{i}$ is set to uniform random values in the range $[0, 1]$ for continuous and binary variables in the \textsc{DSLOTZ} and in the range $[0, 10]$ for continuous and integer variables in the \textsc{DSInt}. The optimization is stopped when the iteration $t$ is reached $10^{4}N$. As with Section~\ref{sec:alpha_search}, the hyperparameter $\alpha$ should depend on the number of dimensions $N$ and sample size $\lambda$. Unfortunately, unlike the original CMA-ES, MO-CMA-ES does not have a recommended value for the sample size. Therefore, we use three different sample sizes $\lambda \in \{10, 50, 100 \}$. The number of dimensions of $N$ is set to 10, 20 or 30, and $N_{\mathrm{co}} = N_{\mathrm{bi}} = N/2$. In this experiment, we evaluate hypervolume after optimization of $\alpha = N^{-m}\lambda^{-m}~(m,n \in \{0, 0.5, 1, 1.5, 2, 2.5, 3 \})$ except $\alpha = 1$. In each setting, 100 trials are performed independently using different seed values.

\begin{figure}[tbhp]
    \centering
    \includegraphics[width=\linewidth]{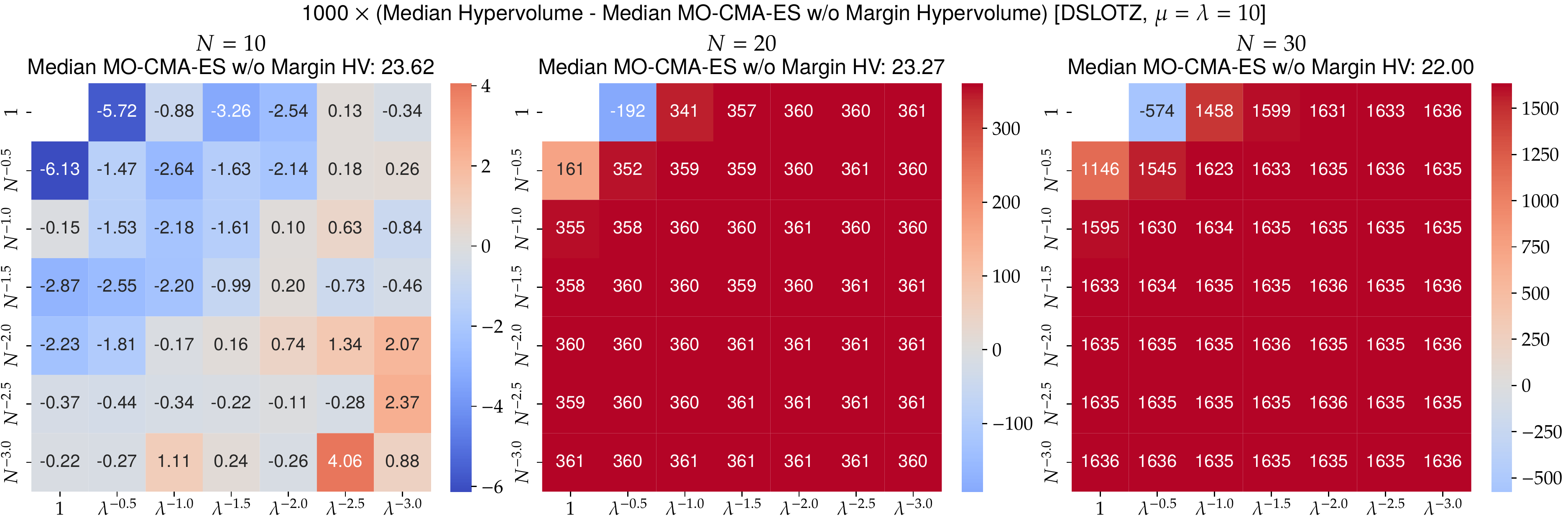}\vspace{5mm}
    \centering
    \includegraphics[width=\linewidth]{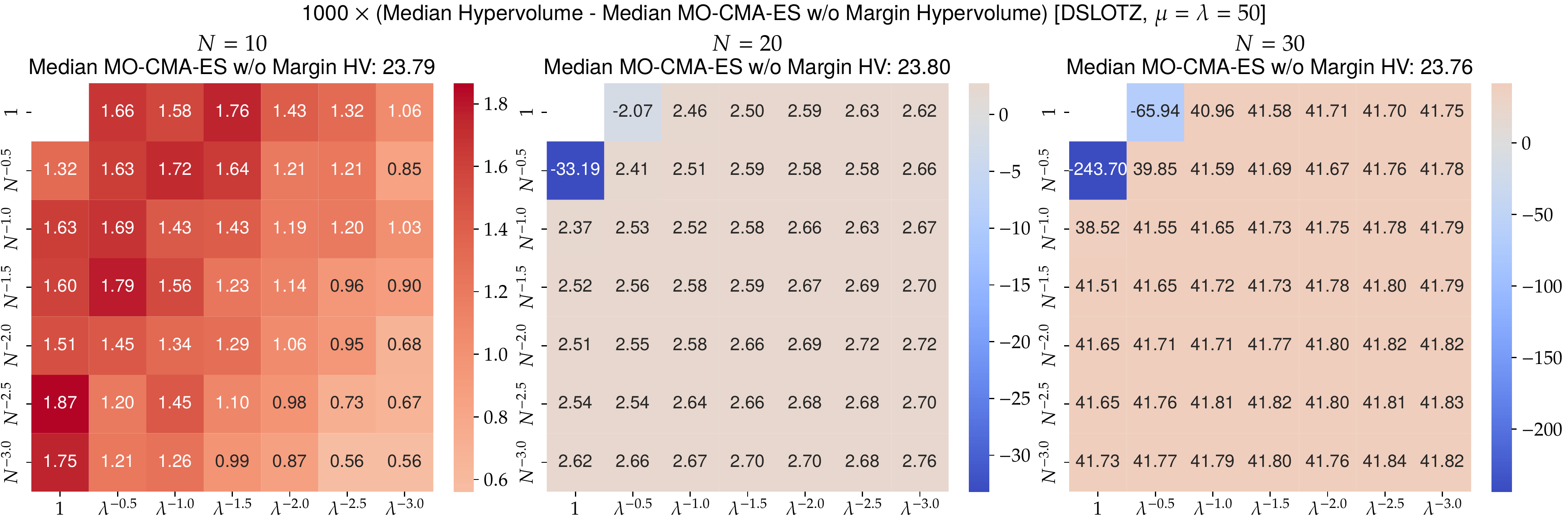}\vspace{5mm}
    \centering
    \includegraphics[width=\linewidth]{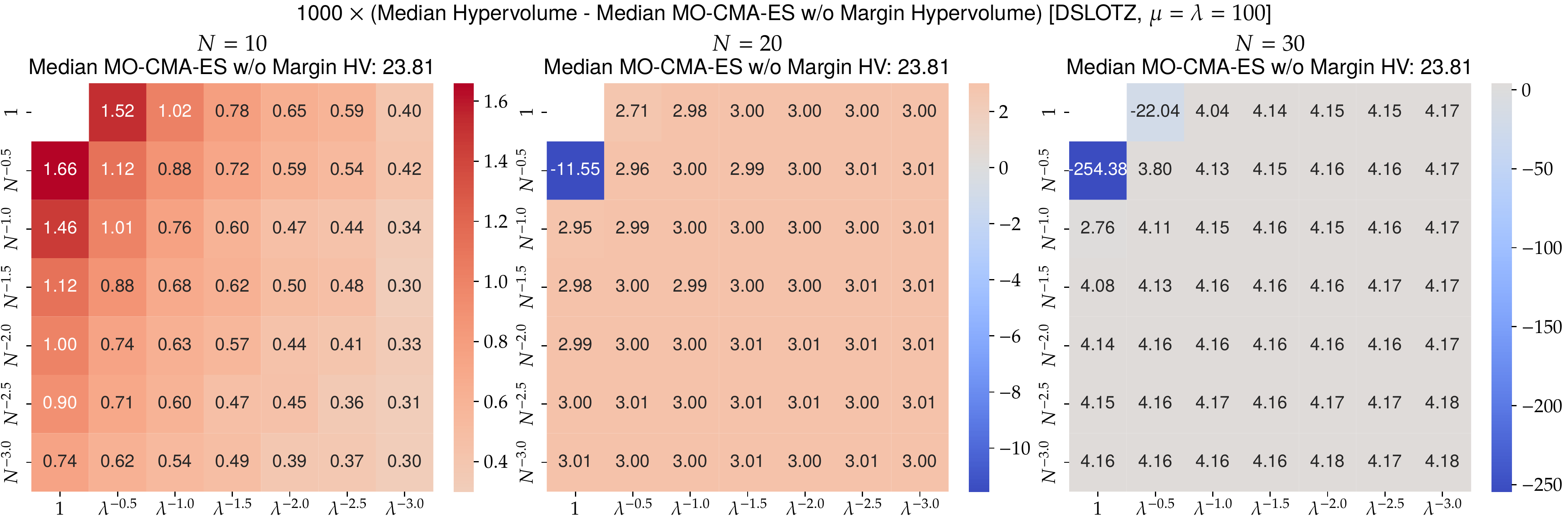}\vspace{5mm}
    \caption{Heatmap of difference in median hypervolume after optimization between MO-CMA-ES with Margin and MO-CMA-ES without margin in the $N$-dimensional \textsc{DSLOTZ} function with $\mu = \lambda = 10~ \textrm{(top)}, 50~\textrm{(middle)}, 100 ~\textrm{(bottom)}$ when the hyperparameter $\alpha=N^{-m}\lambda^{-n}$ of the MO-CMA-ES with Margin is changed. }
    \label{fig:hv_heatmap_dslotz}
\end{figure}

\begin{figure}[tbhp]
    \centering
    \includegraphics[width=\linewidth]{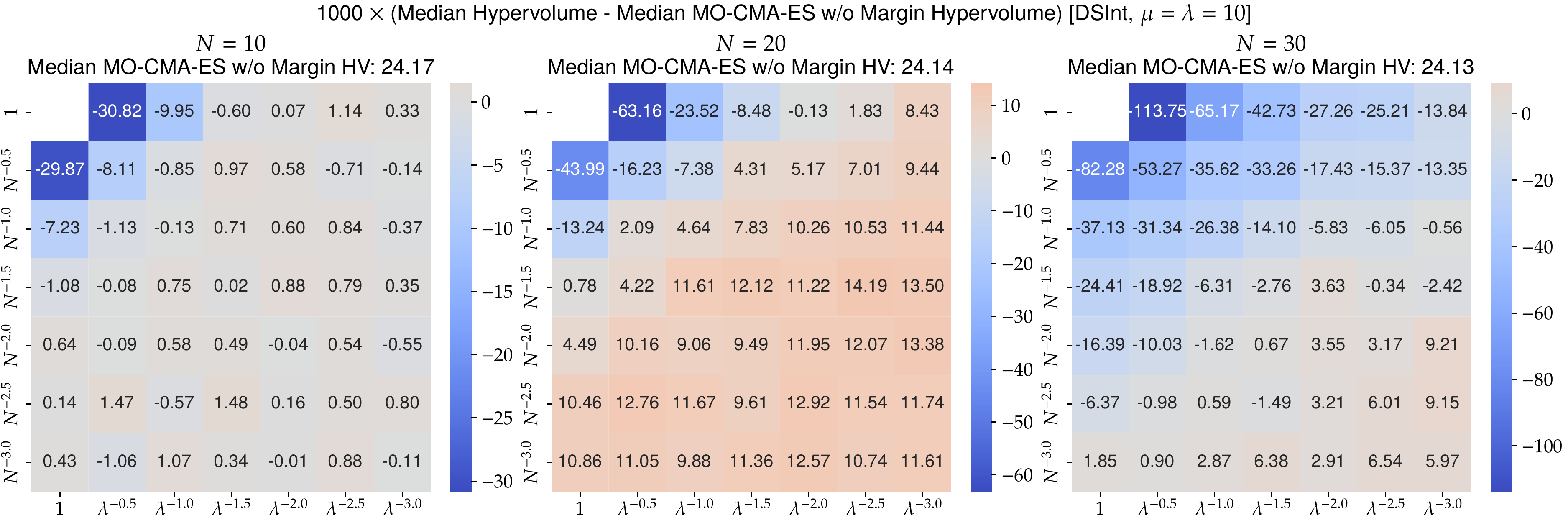}\vspace{5mm}
    \centering
    \includegraphics[width=\linewidth]{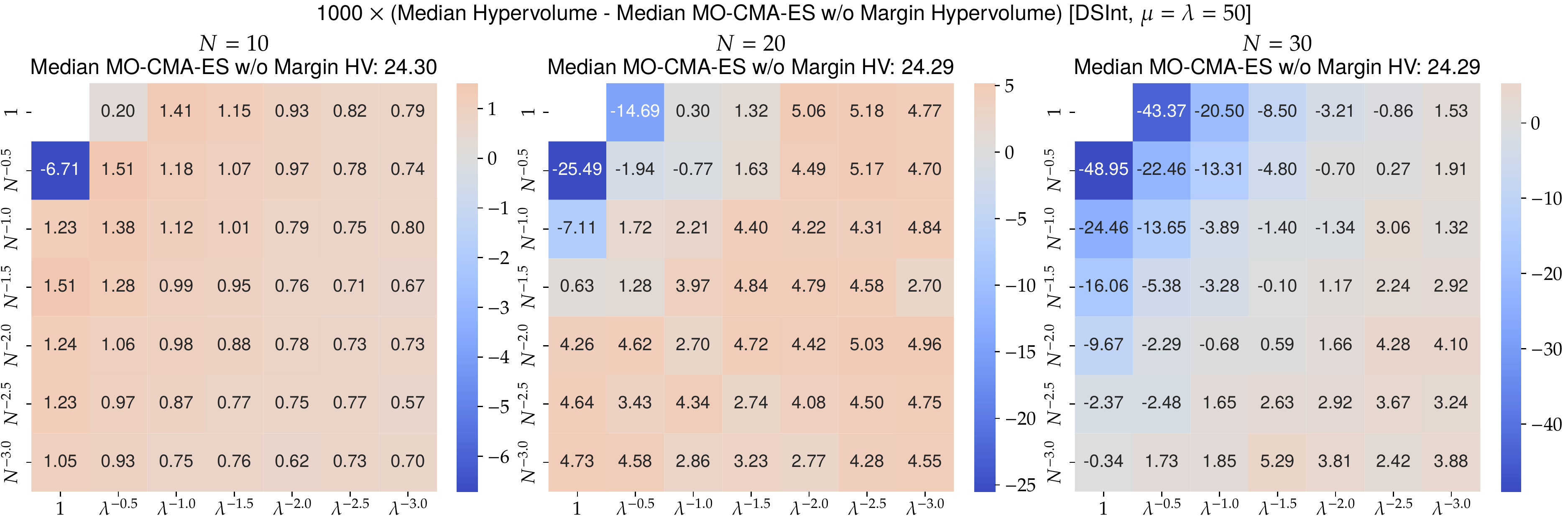}\vspace{5mm}
    \centering
    \includegraphics[width=\linewidth]{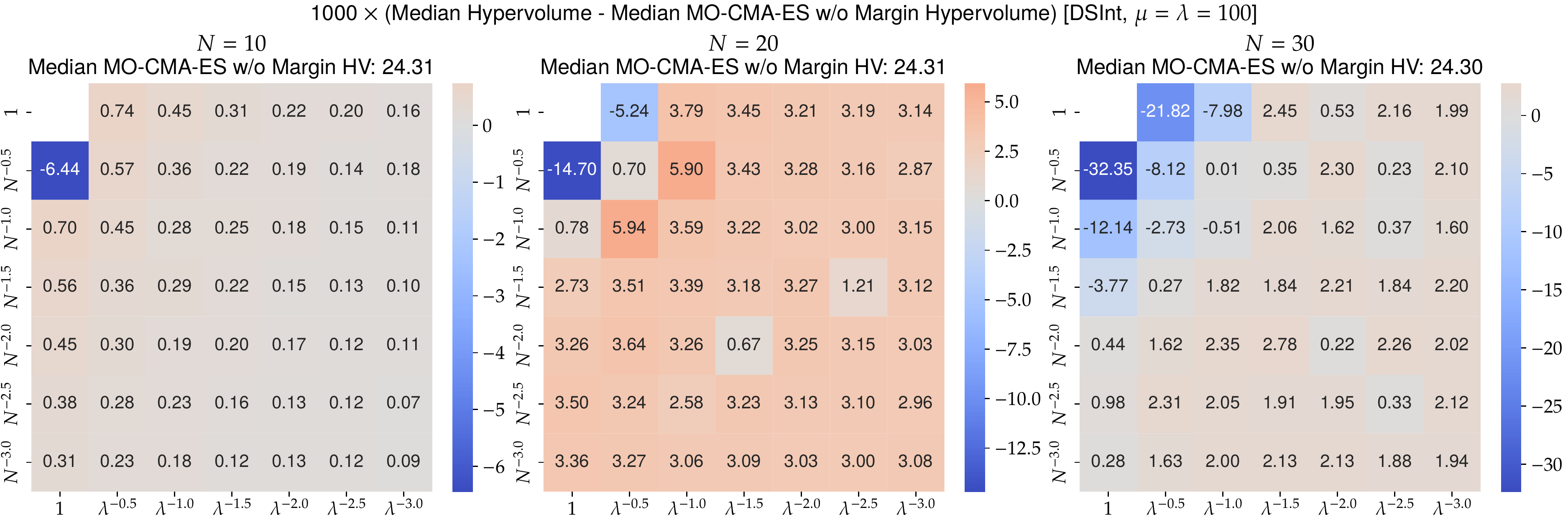}\vspace{5mm}
    \caption{Heatmap of difference in median hypervolume after optimization between MO-CMA-ES with Margin and MO-CMA-ES without margin in the $N$-dimensional \textsc{DSInt} function with $\mu = \lambda = 10~ \textrm{(top)}, 50~\textrm{(middle)}, 100 ~\textrm{(bottom)}$ when the hyperparameter $\alpha=N^{-m}\lambda^{-n}$ of the MO-CMA-ES with Margin is changed.}
    \label{fig:hv_heatmap_dsint}
\end{figure}

\paragraph{Results and Discussion}
Figure~\ref{fig:hv_heatmap_dslotz} shows the difference in median hypervolume for \textsc{DSLOTZ} function after optimization between MO-CMA-ES with Margin and MO-CMA-ES without margin for each setting. 

For $\mu = \lambda=10$ (top in Figure ~\ref{fig:hv_heatmap_dslotz}) and $N = 10$, there is no significant difference between the median hypervolumes of MO-CMA-ES and MO-CMA-ES with Margin, but MO-CMA-ES with Margin has the worst performance when $\alpha \in \{N^{-0.5}, \lambda^{-0.5} \}$. For $\mu = \lambda = 10$ and $N \in \{ 20, 30 \}$, MO-CMA-ES with Margin had worse median hypervolume than MO-CMA-ES when $\alpha = \lambda^{-0.5}$, but other settings improved it by 0.1 $\sim$ 0.3 for $N = 20$ and by more than 1 for $N = 30$.

For $\mu = \lambda \in \{ 50, 100 \}$ (middle and bottom in Figure ~\ref{fig:hv_heatmap_dslotz}) and $N = 10$, there is no significant difference between the median hypervolumes of MO-CMA-ES and MO-CMA-ES with Margin as same as the case of $\mu = \lambda = 10$. For $\mu = \lambda \in \{ 50, 100 \}$ and $N \in \{ 20, 30 \}$, MO-CMA-ES with Margin for $\alpha \in \{N^{-0.5}, \lambda^{-0.5} \}$ except for $N = 20$ has a worse median hypervolumes than MO-CMA-ES, but in other settings, MO-CMA-ES with Margin performs competitively with MO-CMA-ES.

\begin{figure}[t]
    \centering
    \includegraphics[width=\linewidth]{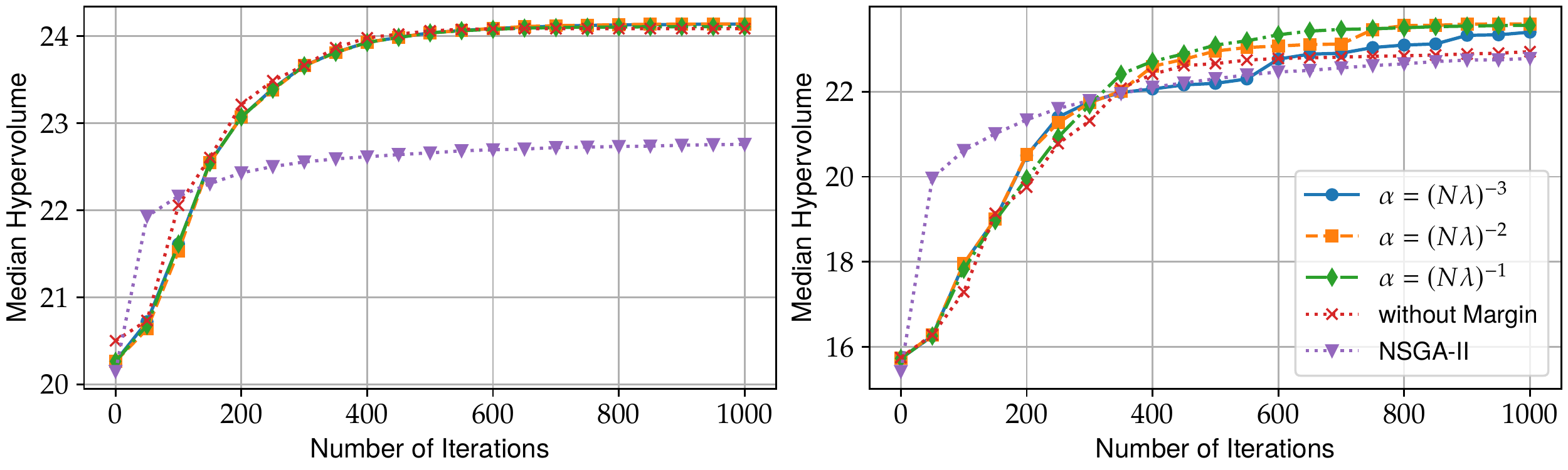}
    \caption{Transition of median hypervolumes of the first 1000 iterations on \textsc{DSInt} (left) and \textsc{DSLOTZ} (right) functions for MO-CMA-ES with and without Margin, and NSGA-\II.}
    \label{fig:hv_transition}
\end{figure}

Next, Figure~\ref{fig:hv_heatmap_dsint} shows the difference in median hypervolume for \textsc{DSInt} function after optimization between MO-CMA-ES with Margin and MO-CMA-ES without margin for each setting. The \textsc{DSInt} function also shows that MO-CMA-ES with Margin for $\alpha\in \{ N^{-0.5}, \lambda^{-0.5}\}$ has the worst hypervolume after optimization, but other settings, MO-CMA-ES with Margin performs competitively with MO-CMA-ES.

In summary, MO-CMA-ES with Margin except in the case of $\alpha\in \{N^{-0.5}, \lambda^{-0.5}\}$ can obtain the well or higher hypervolume than MO-CMA-ES. In particular, the smaller the sample size and the higher the dimension settings, the better the effect of the margin. As with the single objective, a large value such as $\alpha \in \{ N^{-0.5}, \lambda^{-0.5}\}$ is considered to have worsened the hypervolume compared to MO-CMA-ES because the search is unstable. On the other hand, for the other settings, there is no significant difference in the results of MO-CMA-ES with Margin. 

Therefore, we compare the speed of improvement of the hypervolumes for different $\alpha$ settings to determine a recommended value for $\alpha$.
Additionally, we also validate the NSGA-\II~\cite{NSGA2:2002} with the integer handling as implemented in pymoo\footnote{https://pymoo.org/}. In the NSGA-\II, the population size is set to $10$, and the initial population is generated from a uniform distribution in the same range as the case of MO-CMA-ES for fair comparison. The box constraint of range $[-5,5]$ is imposed on continuous variables in the \textsc{DSLOTZ}, and the box constraint of range $[-20,20]$ is imposed on continuous and integer variables in the \textsc{DSInt}.

Figure~\ref{fig:hv_transition} shows the comparison of the transition of median hypervolumes of the first 1000 iterations on \textsc{DSInt} and \textsc{DSLOTZ} functions of $N = 30$ for MO-CMA-ES with Margin at $\alpha \in \{ (N\lambda)^{-3}, (N\lambda)^{-2}, (N\lambda)^{-1} \}$, MO-CMA-ES, and MO-CMA-ES with and without Margin, and NSGA-\II.
The \textsc{DSInt} function shows no significant difference in hypervolume transition for different values of $\alpha$. The reason for this is the ease of obtaining a Pareto set of the \textsc{DSInt} function. In single-objective \textsc{SphereInt} function optimization, if the integer variable in any one dimension was fixed at a non-optimal integer, it was difficult to obtain an optimal solution. In the \textsc{DSInt} function, on the other hand, even if an integer variable in one dimension is fixed, one of the Pareto-optimal solutions can be obtained by optimizing the other integer variables.
On the other hand, the final hypervolume in the NSGA-{\II} is lower than that in the MO-CMA-ES. We observed that some of the final solutions took integers outside the range $[0, 10]$.

\begin{figure}[t]
    \centering
    \includegraphics[width=\linewidth]{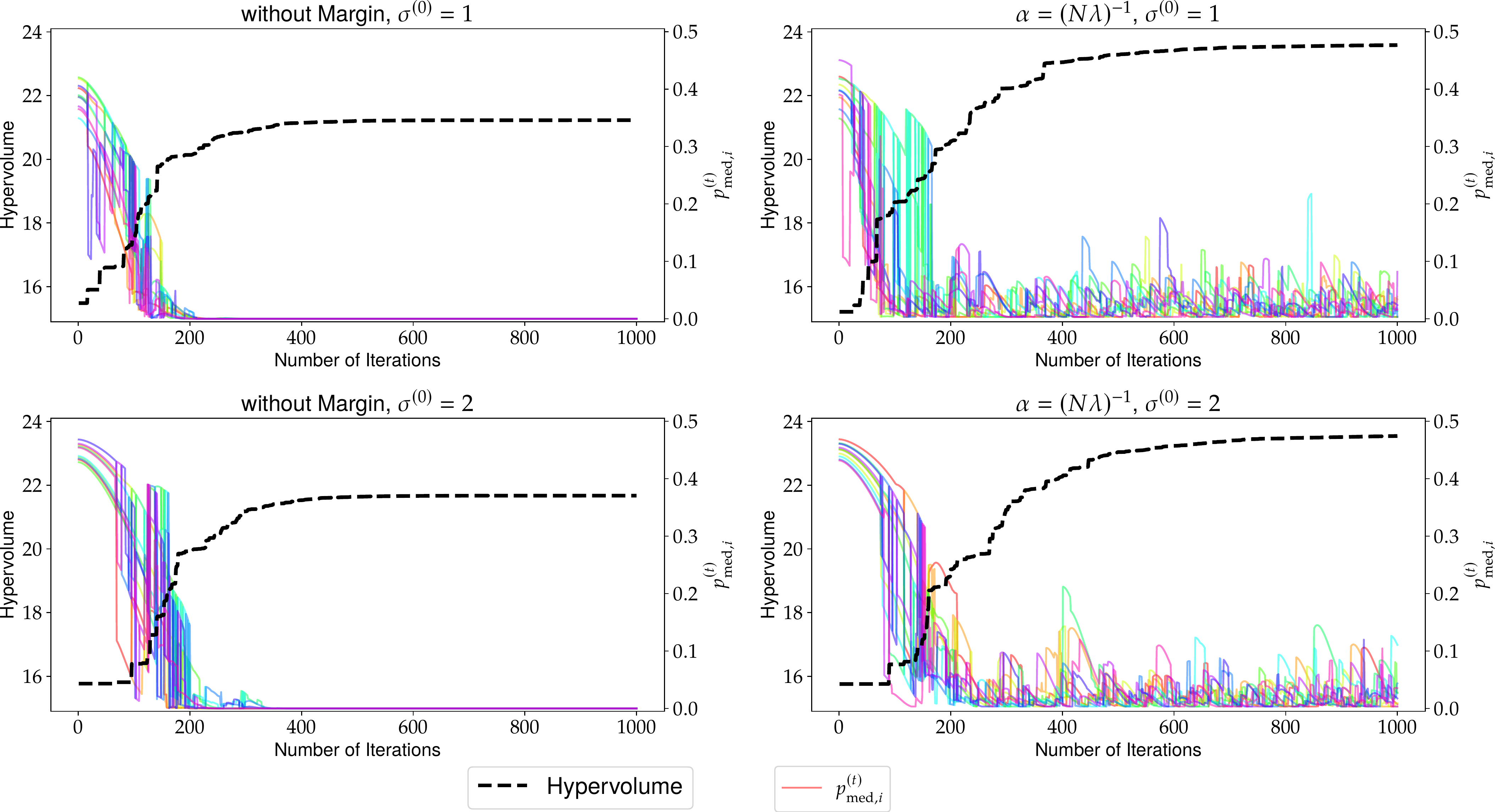}
    \caption{Transition of the hypervolume and $p_{ \mathrm{med}, i}^{(t)}$ ($i = 1, \ldots, \lambda$) in a typical run on $30$-dimensional \textsc{DSLOTZ} for MO-CMA-ES without Margin (left) and with Margin (right) when the initial step-size $\sigma^{(0)}$ is set to $1$ (top) and $2$ (bottom). The population size is set to $\mu = \lambda = 10$.}
    \label{fig:hv_prob_transition}
\end{figure}

On the other hand, in the \textsc{DSLOTZ} function, the difference in hypervolume transition is observed after the 400 iterations.
The MO-CMA-ES without Margin shows optimization has almost stalled. This is because the optimization of the binary variables is not working well due to fixation.
For $\alpha = (N\lambda)^{-3}$, the hypervolume increases slowly, but after 1000 iterations, it reaches about the same hypervolume as the other $\alpha$ settings. For $\alpha \in \{(N\lambda)^{-2}, (N\lambda)^{-1}\}$, hypervolumes of both settings increase at about the same rate, but $\alpha = (N\lambda)^{-1}$ shows a smoother increasing in hypervolume.
Figure~\ref{fig:hv_prob_transition} shows the transition of the hypervolume and $p_{\mathrm{med}, i}^{(t)}$, where $p_{ \mathrm{med}, i}^{(t)}$ is the median of
\begin{align*}
    \left\{ \min \bigl\{ \Pr([{\bar{\boldsymbol{v}'}}_i^{(t+1)}]_j = 0), \Pr([{\bar{\boldsymbol{v}'}}_i^{(t+1)}]_j = 1) \bigr\} ~ \middle| ~ j = N_\mathrm{co}+1, \ldots, N \right\} \enspace.
\end{align*}
Note that $\Pr([{\bar{\boldsymbol{v}'}}_i^{(t+1)}]_j = 0)$ and $\Pr([{\bar{\boldsymbol{v}'}}_i^{(t+1)}]_j = 1)$ are the probabilities that the $i$-th new solution has $0$ and $1$ in the $j$-th dimension in the $t$-th iteration, respectively. If $p_{\mathrm{med}, i}^{(t)}$ is close to $0$, it means that in the $i$-th individual, the probability distribution converges, and the binary variables suffer from fixation.
We can see that without the margin correction, the improvement in the hypervolume stagnates as the binary variables become fixed. This stagnation also occurs when using the large initial step size $\sigma^{(0)} = 2$.
Moreover, we found by experiments that the solution set obtained by the MO-CMA-ES without Margin and the NSGA-\II contained inferior solutions such as ``$11010...0$''.
On the other hand, with the margin correction, the probability distribution does not converge completely, and the hypervolume continues to improve.

Based on these results, we recommend $\alpha = (N\lambda)^{-1}$ for MO-CMA-ES with Margin due to its robustness and efficiency. This recommendation can be shared with the single-objective case and will be enjoyed by users because it reduces the need for tedious parameter tuning.

\section{Conclusion}
\label{sec:conclusion}
In this work, we first experimentally confirmed that the existing integer handling method, CMA-ES-IM~\cite{hansen_cma-es_2011} with or without the box constraint, does not work effectively for binary variables, and then proposed a new integer variable handling method for CMA-ES. In the proposed method, the mean vector and the diagonal affine transformation matrix for the covariance matrix are corrected so that the marginal probability for an integer variable is lower-bounded at a certain level, which is why the proposed method is called the CMA-ES with Margin; it considers both the binary and integer variables.
To demonstrate the generality of the idea of the proposed method, in addition to the single-objective optimization case, we also develop multi-objective CMA-ES with Margin (MO-CMA-ES with Margin), which is aimed at multi-objective mixed-integer optimization.
We confirmed by experiment the behavior of the MO-CMA-ES with Margin on \textsc{DoubleSphereLeadingOnesTailingZeros} and \textsc{DoubleSphereInt}, which we used as bi-objective mixed-integer benchmark problems.

The proposed method has a hyperparameter, $\alpha$, which determines the degree of the lower bound for the marginal probability. We investigated the change in the optimization performance with multiple $\alpha$ settings in order to determine the default parameter. With the recommended value of $\alpha$, we experimented the proposed method on several (MO-)MI-BBO benchmark problems.
The experimental results demonstrated that in the single-objective optimization case, the proposed method is robust even when the number of dimensions increases and can find the optimal solution with fewer evaluations than the existing method, CMA-ES-IM with or without the box constraint.
In the multi-objective optimization case, we first demonstrated the validity of the margin in the MO-CMA-ES, which employs the elitist evolution strategy.
Compared to the original MO-CMA-ES and the NSGA-\II~with the integer handling, MO-CMA-ES with Margin achieved better or at least competitive performance in the experiments with most of $\alpha$ settings.
The effect of the margin was clear, especially with the small sample size and/or the high-dimensional settings.
As a result, the proposed integer handling is effective for single-and multi-objective optimization, as well as for non-elitist and elitist strategy CMA-ES.

There are still many challenges left for the MI-BBO; for example, \citet{tusar_mixed-integer_2019} pointed out the difficulty of optimization for non-separable ill-conditioned convex-quadratic functions, such as the rotated Ellipsoid function. In the future, we need to address these issues, which have not yet been addressed by the proposed or existing methods, by considering multiple dimension correlations.
In addition to using the MGD with continuous relaxation, there is another approach to using the joint distribution of the MGD and discrete distributions such as the Bernoulli distribution. This approach requires the appropriate update balance between the MGD and the discrete distribution, and is left for future work.
Additionally, evaluating the proposed method on real-world MI-BBO problems, including multi-objective optimization, is also an important future direction.

%%
%% The acknowledgments section is defined using the "acks" environment
%% (and NOT an unnumbered section). This ensures the proper
%% identification of the section in the article metadata, and the
%% consistent spelling of the heading.
\begin{acks}
% The authors thank anonymous reviewers for their helpful comments.
This work was partially supported by JSPS KAKENHI Grant Number JP20H04240 and JST PRESTO Grant Number JPMJPR2133. A portion of this paper was based on the results obtained from the project JPNP18002 commissioned by the New Energy and Industrial Technology Development Organization (NEDO).
\end{acks}

%%
%% The next two lines define the bibliography style to be used, and
%% the bibliography file.
\bibliographystyle{ACM-Reference-Format}
\bibliography{ref}

%%
%% If your work has an appendix, this is the place to put it.
% \appendix
% \input{ACM_TELO/manuscripts/99_appendix}

\end{document}